 \documentclass[final, 12pt]{colt2019} 

\usepackage{times}
\usepackage{booktabs}       
\usepackage{amsfonts}       
\usepackage{nicefrac}       
\usepackage{enumitem}

\usepackage{amsmath,amssymb,mathrsfs}
\PassOptionsToPackage{ruled, vlined, linesnumbered, noend}{algorithm2e}
\usepackage{algorithm}
\usepackage{footnote}

\allowdisplaybreaks

\usepackage{graphicx}
\usepackage{mathtools}
\usepackage{footnote}
\usepackage{float}
\usepackage{xspace}
\usepackage{multirow}
\usepackage{wrapfig}
\usepackage{framed}

\usepackage{footnote}

\makesavenoteenv{tabular}
\makesavenoteenv{table}

\usepackage[]{color-edits}
\addauthor{HL}{red}
\addauthor{CW}{olive}
\addauthor{CL}{orange}

\newcommand{\LineComment}[1]{\hfill$\rhd\ $\text{#1}}

\newcommand{\TVD}[1]{\norm{#1}_\text{TV}}

\newcommand{\EG}{\textsc{Epoch-Greedy}\xspace}

\newcommand{\minimonster}{\textsc{ILOVETOCONBANDITS}\xspace}

\newcommand{\AdaEG}{\textsc{Ada-Greedy}\xspace}
\newcommand{\AdaILTCB}{\textsc{Ada-ILTCB}\xspace}

\newcommand{\AdaBIN}{\textsc{Ada-BinGreedy}\xspace}
\newcommand{\algo}{$\textsc{Ada-ILTCB}^+$\xspace}

\newcommand{\bistro}{\textsc{BISTRO+}\xspace}

\newcommand{\term}[1]{\textsc{Term$_{#1}$}}
\newcommand{\termprime}[1]{\overline{\textsc{Term}}\textsc{$_{#1}$}}
\newcommand{\event}[1]{\textbf{\textsc{Event$_{#1}$}}}

\newcommand{\calA}{{\mathcal{A}}}
\newcommand{\calB}{{\mathcal{B}}}
\newcommand{\calX}{{\mathcal{X}}}
\newcommand{\calS}{{\mathcal{S}}}
\newcommand{\calF}{{\mathcal{F}}}
\newcommand{\calI}{{\mathcal{I}}}
\newcommand{\calJ}{{\mathcal{J}}}
\newcommand{\calK}{{\mathcal{K}}}
\newcommand{\calD}{{\mathcal{D}}}
\newcommand{\calE}{{\mathcal{E}}}
\newcommand{\calR}{{\mathcal{R}}}

\newcommand{\avgR}{\wh{\cal{R}}}
\newcommand{\ips}{\wh{r}}
\newcommand{\whpi}{\wh{\pi}}

\newcommand{\whV}{\wh{V}}
\newcommand{\Reg}{\text{\rm Reg}}
\newcommand{\whReg}{\wh{\text{\rm Reg}}}

\newcommand{\one}{\boldsymbol{1}}
\newcommand{\var}{\Delta}
\newcommand{\bvar}{\bar{\Delta}}

\DeclareMathOperator*{\argmax}{argmax}

\DeclarePairedDelimiter\abs{\lvert}{\rvert}
\DeclarePairedDelimiter\bigabs{\big\lvert}{\big\rvert}

\newcommand{\field}[1]{\mathbb{#1}}

\newcommand{\fR}{\field{R}}

\newcommand{\E}{\field{E}}
\newcommand{\testblock}{\textsc{EndofBlockTest}\xspace}
\newcommand{\testreplay}{\textsc{EndofReplayTest}\xspace}

\newcommand{\norm}[1]{\left\|{#1}\right\|}

\newcommand{\wh}{\widehat}

\newcommand{\Kprime}{\overline{K}}

\newtheorem*{main}{Main Result}
\newtheorem{fact}[theorem]{Fact}

\newcommand{\order}{\ensuremath{\mathcal{O}}}
\newcommand{\otil}{\ensuremath{\widetilde{\mathcal{O}}}}

\title[A New Algorithm for Non-stationary Contextual Bandits]{A New Algorithm for Non-stationary Contextual Bandits: \\ Efficient, Optimal, and Parameter-free}

\coltauthor{%
 \Name{Yifang Chen} \Email{yifang@usc.edu}\\
 \Name{Chung-Wei Lee} \Email{leechung@usc.edu}\\
 \Name{Haipeng Luo} \Email{haipengl@usc.edu} \\
 \Name{Chen-Yu Wei} \Email{chenyu.wei@usc.edu}\\
 \addr University of Southern California
}

\begin{document}
\maketitle

\begin{abstract}

We propose the first contextual bandit algorithm that is parameter-free, efficient, and optimal in terms of dynamic regret.
Specifically, our algorithm achieves $\order(\min\{\sqrt{KST}, K^{\frac{1}{3}}\Delta ^{\frac{1}{3}}T^{\frac{2}{3}}\})$ dynamic regret 
for a contextual bandit problem with $T$ rounds, $K$ actions, $S$ switches and $\Delta$ total variation in data distributions.
Importantly, our algorithm is adaptive and does not need to know $S$ or $\Delta$ ahead of time,
and can be implemented efficiently assuming access to an ERM oracle.

Our results strictly improve the $\order (\min \{S^{\frac{1}{4}}T^{\frac{3}{4}}, \Delta^{\frac{1}{5}}T^{\frac{4}{5}}\})$ bound of~\citep{LuoWA018}, 
and greatly generalize and improve the $\order(\sqrt{ST})$ result of~\citep{AuerGO018} that holds only for the two-armed bandit problem without contextual information.
The key novelty of our algorithm is to introduce {\it replay phases}, 
in which the algorithm acts according to its previous decisions for a certain amount of time in order to detect non-stationarity
while maintaining a good balance between exploration and exploitation.

\end{abstract}

\section{Introduction}

For online learning problems, a standard performance measure is \emph{static regret}, which compares the total reward of the best fixed policy (or action/arm/expert under different contexts) and that of the algorithm. 
While minimizing static regret makes sense when there exists a fixed policy with large total reward,
it becomes much less meaningful in a non-stationary environment where data distribution is changing over time and no single policy can perform well all the time.

Instead, in this case a more natural benchmark would be to compare the algorithm with the \textit{best sequence of policies}. 
This is formally defined as \textit{dynamic regret}, 
which is the difference between the total reward of the best sequence of policies and the total reward of the algorithm. 
Due to the ubiquity of non-stationary data, 
there is an increasing trend of designing online algorithms with strong dynamic regret guarantee. 
We provide a more detailed review of related work in Section~\ref{sec:related_work}.
In short, while obtaining dynamic regret is relatively well-studied in the full-information setting,
for the more challenging bandit feedback, most existing works only focus on the simplest multi-armed bandit problem.
More importantly, a sharp contrast between these two regimes is that except for the recent work of~\citep{AuerGO018} for a two-armed bandit problem,
none of the others achieves optimal dynamic regret {\it without the knowledge of the non-stationarity of the data} in the bandit setting,
indicating the extra challenge of being adaptive to non-stationary data with partial information.


In this work, we make a significant step in this direction.
Specifically we consider the general contextual bandit setting~\citep{AuerCeFrSc02, LangfordZh08} which subsumes many other bandit problems.
For an environment with $T$ rounds where at each time $t$ the data is generated from some distribution $\calD_t$,
denote by $S=1+\sum_{t=2}^T \one\{\calD_{t}\neq \calD_{t-1}\}$ the number of switches (plus one) 
and by $\var = \sum_{t=2}^T \TVD{\calD_t-\calD_{t-1}}$ the total variation of these distributions
(see Section~\ref{sec:setup} for more formal definition of the setting) .
Our main contribution is to propose an algorithm called \algo with the following guarantee:

\begin{main}
\algo achieves the optimal dynamic regret bound $\order\left(\min\left\{\sqrt{ST}, \Delta^{\frac{1}{3}}T^{\frac{2}{3}}\right\}\right)$
without knowing $S$ or $\Delta$. Moreover, \algo is oracle-efficient.
\end{main}

Here the dependence on all other parameters are omitted (see Theorem~\ref{theorem: main theorem} for the complete version)
and the optimality of the dependence on $S, \Delta$ and $T$ are well-known \citep{garivier2011upper, BesbesGuZe14}. 
Oracle-efficiency refers to efficiency assuming access to an ERM oracle, 
a common assumption made in most prior works for efficient contextual bandit (formally defined in Section~\ref{sec:setup}).

Our result is by far the best and most general dynamic regret bound for bandit problems.
Recent work by~\citet{LuoWA018} studies the exact same setting and achieves the same optimal bound only if $S$ and $\Delta$ are known;
otherwise their algorithms only achieve suboptimal bounds such as 
$\order (\min \{S^{\frac{1}{4}}T^{\frac{3}{4}}, \Delta^{\frac{1}{5}}T^{\frac{4}{5}}\})$.
On the other hand, \citet{AuerGO018} propose the first bandit algorithm with expected regret $\order(\sqrt{ST})$ without knowing $S$,
but only for the simplest setting: the two-armed bandit problem without contexts.
In contrast, our algorithm works for the general multi-armed bandit problem with contextual information,
enjoys a meaningful bound as long as $\Delta$ is small (even when $S$ is $\order(T)$), 
works with high probability, 
and importantly is oracle-efficient as well.

Our key technique is inspired by~\citep{AuerGO018}.
The high level idea of their algorithm is to occasionally enter some pure exploration phase in order to detect non-stationarity,
and crucially the durations of these exploration phases are \textit{multi-scale} and determined in some randomized way.
The reason behind this is that smaller non-stationarity requires more time to discover and vice versa.
%
We extend this multi-scale idea to the contextual bandit setting.
However, the extension is highly non-trivial and requires the following two new elements:

\begin{enumerate}
\item
First, we find that pure exploration over arms (used by~\citet{AuerGO018, LuoWA018}) is not the optimal way to detect non-stationarity in contextual bandit. 
Instead, we propose to let the algorithm occasionally enter \textit{replay phases},
meaning that the algorithm acts according to some policy distribution used earlier by the algorithm itself. 
The duration of a replay phase and which previous policy distribution to replay are both determined in some randomized way similar to~\citep{AuerGO018}. 
This can be seen as an interpolation between using the current policy distribution and using pure exploration,
and as shown by our analysis achieves a better trade off between exploitation and exploration in non-stationary environments. 

\item
Second, the algorithm of \citep{AuerGO018} is an ``arm-elimination'' approach, which eliminates arms as long as their sub-optimality is identified. 
Direct extension to contextual bandit leads to an inefficient approach similar to \textsc{PolicyElimination} by~\citet{DudikHsKaKaLaReZh11}.
Instead, our algorithm is based on the \textit{soft elimination scheme} of \citep{AgarwalHsKaLaLiSc14}
and can be efficiently implemented with an ERM oracle.
Combining this soft elimination scheme and the replay idea in a proper way is another key novelty of our work.
\end{enumerate}

We review related work in Section~\ref{sec:related_work} and introduce all necessary preliminaries in Section~\ref{sec:setup}.
Our algorithm is presented in Section~\ref{Section:algorithm}. 
The rest of the paper is dedicated to the relatively involved analysis of our algorithm.

\section{Related Work}\label{sec:related_work}
\paragraph{Different forms of dynamic regret bound.}
Bounding dynamic regret in terms of the number of switches $S$ is traditionally referred to as switching regret or tracking regret,
and has been studied under various settings.
Note, however, that in some works $S$ refers to the number of switches of data distributions just as our definition
(e.g.~\citep{garivier2011upper, WeiHoLu16, liu2018change, LuoWA018}),
while in others $S$ refers to the more general notion of number of switches in the competitor sequence 
(e.g.~\citep{herbster1998tracking, bousquet2002tracking, AuerCeFrSc02, hazan2009efficient}).

Bounding dynamic regret in terms of the variation of loss functions or data distributions is also widely studied
(e.g.~\citep{BesbesGuZe14, BesbesGuZe15, LuoWA018}),
and there are in fact several other forms of dynamic regret bounds studied in the literature
(e.g.~\citep{Zinkevich03, slivkins2008adapting, jadbabaie2015online, WeiHoLu16, yang2016tracking, zhang2017improved}).

\paragraph{Adaptivity to non-stationarity.}
Achieving optimal dynamic regret bounds without any prior knowledge of the non-stationarity is the main focus of this work.
This has been achieved for most full-information problems~\citep{LuoSc15, jun2017online, ZhangYaJiZh18},
but is much more challenging in the bandit setting.
Several recent attempts only achieve suboptimal bounds~\citep{KarninAn16, LuoWA018, cheung2019learning}.
It was not clear whether optimal bounds were achievable in this case,
until the recent work of~\citet{AuerGO018} answers this in the affirmative for the two-armed bandit problem.
As mentioned our results significantly generalize their work.



\paragraph{Contextual bandits.} 
Contextual bandit is a generalization of the multi-armed bandit problem.
While direct generalization of the classic multi-armed bandit algorithm already achieves the optimal static regret~\citep{AuerCeFrSc02},
recent research has been focusing on developing practically efficient algorithms with strong regret guarantee due to their applicability to real-world applications.
To avoid running time that is linear in the size of the policy set,
most existing works make the practical assumption that an ERM oracle is given to solve the corresponding offline problem.
Based on this assumption, a series of progress has been made on developing oracle-efficient algorithms with small static regret~\citep{LangfordZh08, DudikHsKaKaLaReZh11, AgarwalHsKaLaLiSc14, syrgkanis2016efficient, rakhlin2016bistro, SyrgkanisLuKrSc16}.
All these results rely on some stationary assumption of the environment,
since it is known that minimizing static regret oracle-efficiently is impossible in an adversarial environment~\citep{Hazan2016}.


Despite the negative result for static regret with oracle-efficient algorithms, 
\citet{LuoWA018} find that this is no longer true for dynamic regret, and develop oracle-efficient algorithms with optimal dynamic regret when the non-stationarity is known.
Their work is most closely related to ours and our algorithm is in essence similar to their \AdaILTCB algorithm.
The key novelty compared to theirs is the replay phases mentioned earlier, 
which eventually allows the algorithm to adapt to the non-stationarity of the data.

\paragraph{Replay phases.}
Introducing replay phases is one of our key contributions. 
The closest idea in the literature is the method of ``mixing past posteriors'' of~\citep{bousquet2002tracking, adamskiy2012putting},
which at each time acts according to some weighted combination of all previous distributions.
One key difference of our method is that once it enters into a replay phase,
it has to continue for a certain amount of time to gather enough information for non-stationarity detection.
Another difference is that in~\citep{bousquet2002tracking, adamskiy2012putting} the main point of ``mixing past posteriors'' is to obtain some form of ``long-term memory'';
otherwise for typical dynamic regret bounds it is enough to just mix with some amount of pure exploration.
It is not clear to us whether our replay idea actually equips the algorithm with some kind of ``long-term memory'' as well,
and we leave this as a future direction.

\section{Preliminaries}\label{sec:setup}
The contextual bandit problem is defined as follows. Let $\calX$ be some arbitrary context space and $K$ be the number of actions. A policy $\pi: \calX\rightarrow [K]$ is a mapping from the context space to the actions.\footnote{%
Throughout the paper we use the notation $[n]$ to denote the set $\{1, \ldots, n\}$ for some integer $n$.
}  
The learner is given a set of policies $\Pi$,
assumed to be finite for simplicity but with a huge cardinality ${|\Pi|}$.
Before the learning procedure starts, the environment decides $T$ distributions $\calD_1, \ldots, \calD_T$ on $\calX\times [0,1]^K$, and draws $T$ independent samples from them: $(x_t, r_t)\sim \calD_t, \forall t\in[T]$.\footnote{%
Technically $\calD_1, \ldots, \calD_T$ are density functions assumed to be absolutely continuous.
} 
The learning procedure then proceeds as follows: for each time $t = 1, \ldots, T$, the learner first receives the context $x_t$, and then based on this context picks an action $a_t\in [K]$. Afterwards the learner receives the reward feedback $r_t(a_t)$ for the selected action but not others. The instantaneous regret against a policy $\pi$ at time $t$ is $r_t(\pi(x_t))-r_t(a_t)$. The classic goal of contextual bandit algorithms is to minimize $\max_{\pi\in\Pi}\sum_{t=1}^T r_t(\pi(x_t))-r_t(a_t)$, that is, the cumulative regret against the best fixed policy,
and the optimal bound is known to be $\order(\sqrt{KT\ln {|\Pi|}})$ in expectation~\citep{AuerCeFrSc02}.

The classic regret is not a good performance measure for non-stationary environments where no single policy can perform well all the time.
Instead, we consider dynamic regret that compares the reward of the algorithm to the reward of the best policy {\it at each time}.
Specifically, denote the expected reward of policy $\pi$ at time $t$ as $\calR_t(\pi)\triangleq \E_{(x,r)\sim \calD_t}\left[r(\pi(x))\right]$, and the optimal policy at time $t$ as $\pi_t^*\triangleq \argmax_{\pi\in \Pi}\calR_t(\pi)$. The dynamic regret is then defined as $\sum_{t=1}^T r_t(\pi_t^*(x_t)) - r_t(a_t)$. 

It is well-known that in general it is impossible to achieve sub-linear dynamic regret.
Instead, typical dynamic regret bounds are expressed in terms of some quantities that characterize the non-stationarity of the data distributions, and are meaningful as long as these quantities are sublinear in $T$. Two such quantities considered in this work are: the number of distribution hard switches (plus one) $S\triangleq 1+\sum_{t=2}^T \one\{\calD_{t}\neq \calD_{t-1}\}$ and the total variation of distributions $\var \triangleq \sum_{t=2}^T \TVD{\calD_t-\calD_{t-1}} = \sum_{t=2}^T\int_{[0,1]^K}\int_{\calX} \bigabs{\calD_{t}(x,r)-\calD_{t-1}(x,r)} dxdr$. 

\paragraph{More notation.}
For any integer $1\leq s \leq s' \leq T$, we denote by $[s,s']$ the time interval $\{s,s+1, \ldots, s'\}$. For an interval $\calI=[s,s']$, we define the number of switches and the variation on this interval respectively as $S_\calI\triangleq 1+\sum_{\tau=s+1}^{s'}\one\{\calD_{\tau}\neq \calD_{\tau-1}\}$ and $\Delta_\calI\triangleq \sum_{\tau=s+1}^{s'} \TVD{\calD_{\tau}-\calD_{\tau-1}}$. 

As in most algorithms, at each time $t$ we sample an action $a_t$ according to some distribution $p_t$, calculated based on the history before time $t$.
After receiving the reward feedback $r_t(a_t)$, we construct the usual importance-weighted estimator $\ips_t$, which is defined as $\ips_t(a) \triangleq \frac{r_t(a)}{p_t(a)}\one\{a_t=a\}, \forall a\in[K]$ and is clearly unbiased with mean $r_t$. 

For any interval $\calI\subseteq [T]$, we define the average reward of a policy $\pi$ over this interval as $\calR_\calI(\pi) \triangleq \frac{1}{|\calI|}\sum_{t\in\calI} \calR_t(\pi)$ and similarly its empirical average reward as $\avgR_\calI(\pi) \triangleq \frac{1}{|\calI|}\sum_{t\in\calI} \ips_t(\pi(x_t))$. The optimal policy in interval $\calI$ is defined as $\pi_\calI^* \triangleq \argmax_{\pi\in\Pi} \calR_\calI(\pi)$ while the empirically best policy is $\whpi_\calI \triangleq \argmax_{\pi\in\Pi} \avgR_\calI(\pi)$. 
Furthermore, 
the expected and empirical interval (static) regret of a policy $\pi$ for an interval $\calI$ are respectively defined as $\Reg_\calI(\pi) \triangleq \calR_{\calI}(\pi_\calI^*) - \calR_\calI(\pi)$ and $\whReg_\calI(\pi) \triangleq \avgR_\calI(\whpi_\calI) - \avgR_\calI(\pi)$. 
When $\calI = [t, t]$, we simply use $t$ to replace $\calI$ as the subscript. For example, $\Reg_t(\pi)$ represents $\Reg_{[t,t]}(\pi)$.

For a context $x$ and a distribution over the policies $Q \in
\Delta^\Pi \triangleq \{Q \in \fR^{|\Pi|}_+: \sum_{\pi \in \Pi} Q(\pi) =
1\}$, the projected distribution over the actions is denoted by
$Q(\cdot|x)$ such that $Q(a|x) = \sum_{\pi: \pi(x)=a} Q(\pi)$
for all $a \in [K]$. The smoothed projected distribution with a minimum
probability $\nu \in (0, 1/K]$ is defined as $Q^\nu(\cdot|x) = \nu\one +
(1-K\nu)Q(\cdot|x)$ where $\one$ is the all-one
vector. Similarly to~\citep{AgarwalHsKaLaLiSc14}, our algorithm keeps track of a bound
on the variance of the reward estimates. To this end, define for a policy $\pi$, an
interval $\calI$,  a distribution $Q$, and a minimum probability $\nu$,
the empirical and expected variance as
\[
\whV_\calI (Q, \nu, \pi)\triangleq \frac{1}{|\calI|} \sum_{t \in \calI} \left[ \frac{1}{Q^{\nu}(\pi(x_t) | x_t)} \right], 
\quad V_\calI(Q, \nu, \pi) \triangleq \frac{1}{|\calI|} \sum_{t \in \calI} \E_{x \sim \calD_t^\calX} \left[ \frac{1}{Q^{\nu}(\pi(x) | x)} \right],
\]
where $\calD_t^\calX$ is the marginal distribution of $\calD_t$ over the context space $\calX$.
Again, $\whV_t$ and $V_t$ are shorthands for $\whV_{[t, t]}$ and $V_{[t, t]}$ respectively.

We are interested in efficient algorithms assuming access to an ERM oracle~\citep{AgarwalHsKaLaLiSc14}, defined as:

\begin{definition}
An ERM oracle is an algorithm which takes any set $\mathcal{T}$ of
context-reward pairs $(x, r) \in \calX \times \fR^K$ as
inputs and outputs any policy in $\argmax_{\pi \in \Pi} \sum_{(x,r)\in
  \mathcal{T}} r(\pi(x))$.
\end{definition}

An algorithm is oracle-efficient if its total running time and the
number of oracle calls are both polynomial in $T, K$ and $\ln {|\Pi|}$,
excluding the running time of the oracle itself. 

Finally, we use notation $\otil(\cdot)$ to suppress logarithmic dependence on $T, K$, and $1/\delta$
for some confidence level $\delta$.
For notational convenience we also define $\Kprime=(\log_2T)K$.  
A complete notation table can be found in Appendix~\ref{app:notation_table}.

\section{Algorithm}
\label{Section:algorithm}

\begin{savenotes}
\begin{algorithm}[t]
\SetAlgoVlined
 \caption{\algo}
\label{algorithm:main}
\textbf{Input:} confidence level $\delta\in (0,1)$, time horizon $T$, underlying policy class $\Pi$. \\
\textbf{Definition:} 
$\nu_j = \sqrt{\frac{C_0}{K2^{j} L}}$, where $C_0=\ln\left(\frac{8T^3|\Pi|^2}{\delta}\right)$ and 
$L=\lceil 4KC_0 \rceil$ (block base length). \\
$\calB_{(i,j)} \triangleq [\tau_i, \tau_i+2^j L-1]$, 
where $\tau_i$ is the beginning of epoch $i$, as defined in the algorithm. \\
\nl \textbf{Initialize:} $t=1,i=1$. \\ 
\nl $\tau_i \leftarrow t$. \label{line:restart} \LineComment{$i$ indexes an epoch}  \\
\nl \For(\LineComment{$j$ indexes a block}){$j=0,1,2,\ldots$}{ 
\nl      If $j=0$, define $Q_{(i,j)}$ as an 
         arbitrary sparse distribution over $\Pi$;\footnote{%
We emphasize ``sparse distribution'' only to ensure the efficiency of the algorithm. Whether $Q_{(i,j)}$ is sparse or not does not affect the regret bound since we trivially bound the regret for block $0$ by its length.
}
         otherwise, let $Q_{(i,j)}$ be 
         the solution of (OP) with inputs $\calI=\calB_{(i,j-1)}$ and $\nu = \nu_j$. \\
\nl      $\mathcal{S}\leftarrow\emptyset$.  \LineComment{$\calS$ records replay indices and intervals} \label{line: S reset}  \\
\nl      \While{$t\leq \tau_i+2^j L-1 $}{
         \ \\
\nl      \fbox{$\rhd$ {Step 1. Randomly start a replay phase}} \\
\nl      Sample 
         \textsc{rep}$\sim \text{Bernoulli}\left(\frac{1}{L}\times2^{-j/2}\times \sum_{m=0}^{j-1}2^{-m/2} \right)$.  \label{line: start rep}\\
\nl      \If{$\textsc{rep}=1$}{
\nl          Sample $m$ from $\{0, \ldots, j-1\}$
             s.t. $\Pr[m=b]\propto 2^{-b/2}$.\\
\nl          $\calS \leftarrow \calS \cup \left\{\left(m, [t,t+2^m L-1]\right)\right\}$. \LineComment{start a new replay phase} \label{line: end rep}
            
        }
        \ \\ 
\nl     \fbox{$\rhd$  {Step 2. Sample an action}} \\  
\nl     Let 
        $M_t \triangleq \{m \ \vert\  \exists \;\calI \text{\ such that\ } t\in\calI \text{\ and\ } (m, \calI)\in \calS \}$. \label{line: M_t definition} \\
\nl     \If{$M_t$ is empty}{
\nl         Play $a_t\sim Q_{(i,j)}^{\nu_j}(\cdot|x_t)$. 
        }
\nl        \Else{
\nl         Sample $m\sim \text{Uniform}(M_t)$\\
\nl         Play $a_t\sim Q_{(i,m)}^{\nu_m}(\cdot|x_t)$.   
        }
        
        \ \\ 
\nl     \fbox{$\rhd$ {Step 3. Perform non-stationarity tests}}\\ 
\nl     \For{$(m, [s, s']) \in \calS$}{
\nl          \If{$s'=t$ \emph{\textbf{and}}   \label{line: check restart1}
                \testreplay $\left(i,j,m, [s, t]\right)=$ \textit{Fail}}
               {  
\nl            $t\leftarrow t+1$,          
               $i\leftarrow i+1$ and
               \textbf{goto} Line~\ref{line:restart} to start a new epoch.     \label{line: restart1}
            }
        }
\nl     \If{$t=\tau_i+2^jL-1$ \emph{\textbf{and}}       \label{line: check restart2}
               \testblock$\left(i,j\right)=$\textit{Fail}}
            {
\nl            $t\leftarrow t+1$,
               $i\leftarrow i+1$ and
           \textbf{goto} Line~\ref{line:restart} to start a new epoch.      \label{line: restart2}
        }
        \ \\ 
\nl     $t\leftarrow t+1$. 
      }
}
\end{algorithm}
\end{savenotes}

\begin{figure}[ht]
\begin{framed}
\textbf{Optimization Problem (OP)}\\
\textbf{Input:} time interval $\calI$, minimum exploration probability $\nu$ \\
Return $Q \in \Delta^\Pi$ such that for constant $C=1.2\times 10^7$, 
\begin{gather}
    \sum_{\pi \in \Pi} Q(\pi)\whReg_\calI(\pi) \leq 2CK\nu, \label{eqn:op1} \\
    \forall \pi \in \Pi: \whV_\calI(Q,\nu,\pi) \leq 2K + \frac{\whReg_\calI(\pi)}{C\nu}.  \label{eqn:op2}  
\end{gather}
\end{framed}

\begin{framed}
\textbf{\testreplay}$(i, j, m, \calA)$\\
Return \textit{Fail} if there exists $\pi\in\Pi$ such that any of the following inequalities holds: 
\begin{gather}
     \whReg_{\calA}(\pi) - 4\whReg_{\calB_{(i,j-1)}}(\pi) \geq D_1\Kprime\nu_m, \label{eqn:test_cond1}\\
     \whReg_{\calB_{(i,j-1)}}(\pi) - 4\whReg_{\calA}(\pi) \geq D_1\Kprime\nu_m, \label{eqn:test_cond2}\\
     \whV _{\calA}(Q_{(i,m)}, \nu_m,\pi) - 41 \whV_{\calB_{(i,j-1)}}(Q_{(i,m)}, \nu_m,\pi) \geq  D_2K, \label{eqn:test_cond3} 
\end{gather} 
where $D_1=6400$ and $D_2=800$; otherwise return \textit{Pass}.
\end{framed}

\begin{framed}
\textbf{\testblock}$(i,j)$\\
Return \textit{Fail} if there exists $k\in \{0,\ldots, j-1\}$ and $\pi\in\Pi$ such that any of the following inequalities holds: 
\begin{gather}
    \whReg_{\calB_{(i,j)}}(\pi) - 4\whReg_{\calB_{(i,k)}}(\pi) \geq D_4\Kprime\nu_k, \label{eqn:test_cond4}\\
    \whReg_{\calB_{(i,k)}}(\pi) - 4\whReg_{\calB_{(i,j)}}(\pi) \geq D_4\Kprime\nu_k, \label{eqn:test_cond5} \\
     \whV _{\calB_{(i,j)}}(Q_{(i,k+1)}, \nu_{k+1},\pi) - 41 \whV_{\calB_{(i,k)}}(Q_{(i,k+1)}, \nu_{k+1},\pi) \geq  D_5K, \label{eqn:test_cond6}
\end{gather}
where $D_4=6400$ and $D_5=800$; otherwise return \textit{Pass}.  
\end{framed}

\caption{Optimization subroutine and non-stationarity tests}
\label{fig:tests_and_OP}
\end{figure}

Our algorithm is built upon \minimonster of~\citep{AgarwalHsKaLaLiSc14}.
The main idea of their algorithm is to find a sparse distribution over the policies with both low empirical regret and low empirical variance on the collected data,
and then sample actions according to this distribution.
Finding such distributions is formalized in Figure~\ref{fig:tests_and_OP}, Optimization Problem (OP),
and~\citet{AgarwalHsKaLaLiSc14} show that this can be efficiently implemented using an ERM oracle and importantly the distribution is sparse.
Under a stationary environment, 
it can be shown that the empirical regret concentrates around the expected regret reasonably well and thus the algorithm has low regret.

The \AdaILTCB algorithm of \citep{LuoWA018} works by equipping \minimonster with some non-stationarity tests and restarting once non-stationarity is detected.
Our algorithm \algo works under a similar framework with similar tests,
but importantly enters into replay phases occasionally.
The complete pseudocode is included in Algorithm~\ref{algorithm:main}
and we describe in detail how it works below.

The algorithm starts a new {\it epoch} every time it restarts
(that is, on execution of Line~\ref{line: restart1} or~\ref{line: restart2}).
We index an epoch by $i$ and denote the first round of epoch $i$ by $\tau_i$. 
Within an epoch, the algorithm works on a {\it block} schedule.
Specifically, in epoch $i$,
we call the interval $[\tau_i, \tau_i+L-1]$ block $0$ and interval $[\tau_i+2^{j-1}L, \tau_i+2^jL-1]$ block $j$ for any $j\geq 1$ 
(in the case of restart, the block ends earlier),
where $L$ is some fixed base length.\footnote{%
The lengths of these blocks are doubling except that block 0 and block 1 have the same length $L$. This is merely for notational convenience and it is not crucial.
}
%
Each block is associated with an exploration probability $\nu_j$ of order $1/\sqrt{K 2^{j}L}$.
At the beginning of each block $j$ (for $j\geq 1$),
the algorithm first solves the Optimization Problem (OP) (Figure~\ref{fig:tests_and_OP}) using exploration probability $\nu_j$ and all data collected since the beginning of the current epoch,
that is, data from $\calB_{(i,j-1)} \triangleq [\tau_i, \tau_i+2^{j-1} L-1]$.
%
%
The solution is denoted by $Q_{(i,j)}$, which is a sparse distribution over policies. 



Afterwards, for most of the time of the current block, 
the algorithm simply plays according to $Q^{\nu_j}_{(i,j)}(\cdot|x_t)$, just like \minimonster.
The difference is that at each time, with probability $\frac{1}{L}2^{-j/2}2^{-m/2}$ the algorithm enters into a \textit{replay phase} of index $m\in \{0, 1, \ldots, j-1\}$ which lasts for $2^m L$ rounds.
This is implemented in Line~\ref{line: start rep}-\ref{line: end rep}, 
where we first sample a Bernoulli variable \textsc{rep} to decide whether or not to enter into a replay phase, 
and if so then randomly select a replay index $m$ to ensure the aforementioned probability.
The set $\calS$ is used to record all pairs of replay index and replay interval.
Similar to~\citep{AuerGO018},
the reason of using different lengths is to allow the algorithm to detect different level of non-stationarity:
a longer replay interval with a larger index is used to detect smaller non-stationarity.

Note that at each time $t$, 
the algorithm could potentially be in multiple replay phases simultaneously.
Let $M_t$ be the set of indices of all the \textit{ongoing} replay intervals (defined in Line~\ref{line: M_t definition}). 
If $M_t$ is empty, the algorithm is not in any replay phase and simply samples an action according to $Q^{\nu_j}_{(i,j)}(\cdot|x_t)$ as mentioned. 
On the other hand, if $M_t$ is not empty, 
the algorithm uniformly at random picks an index $m$ from $M_t$, 
and then replays the distribution learned at the beginning of block $m$, that is, samples an action according to $Q^{\nu_m}_{(i,m)}(\cdot|x_t)$.
Recall that our reward estimators $\ips_t$'s are defined in terms of a distribution $p_t$ over actions, 
and it is clear that for our algorithm $p_t(\cdot) = \one\{|M_t|=0\} Q^{\nu_j}_{(i,j)}(\cdot|x_t) +  \one\{|M_t|\neq0\} \frac{1}{|M_t|}\sum_{m\in M_t}Q^{\nu_m}_{(i,m)}(\cdot|x_t)$.

\begin{figure}
    \includegraphics[width=15cm]{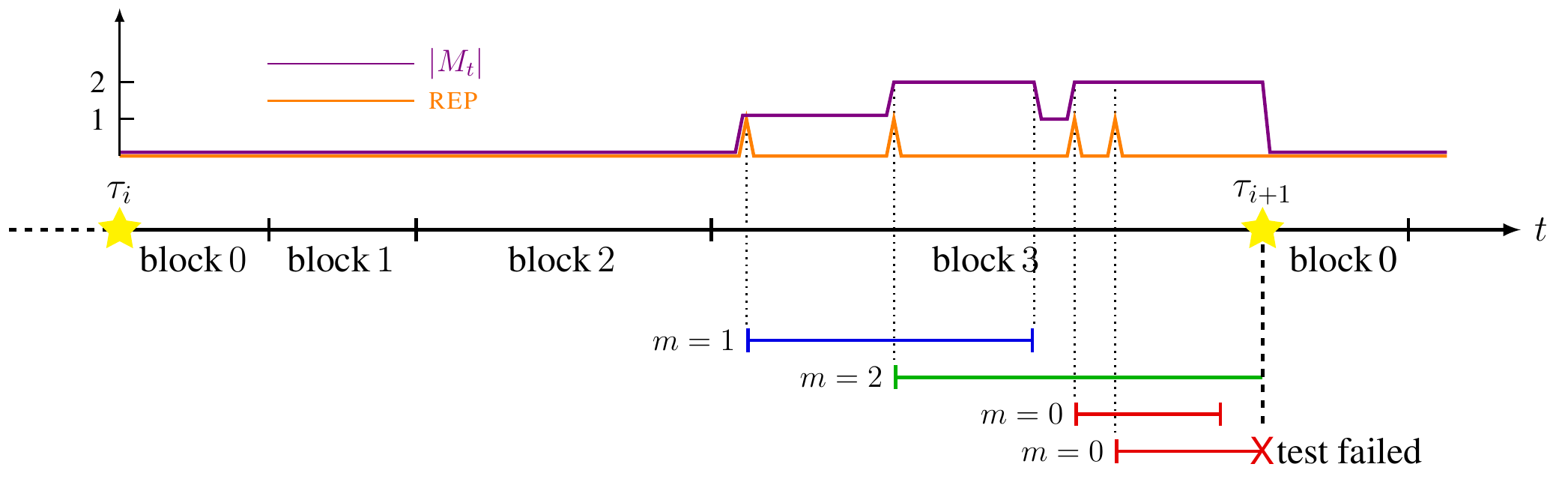}\caption{An Illustration of \algo (best viewed in color). 
The two stars represent two restarts and the interval between them is epoch $i$.
The purple curve represents the value of $|M_t|$, that is, 
how many replay phases (with distinct indices) the algorithm is currently in.
The orange curve represents the value of the Bernoulli variable \textsc{rep},
so that each spike represents the start of a new replay interval.
The four segments below the $t$-axis indicates four replay intervals started within block 3,
with their index value $m$ on the left.
Segments with the same index share the same color,
and segment with larger index is longer.
At time $\tau_{i+1}-1$, the bottom replay interval finishes
and the algorithm performs a \testreplay.
The test fails and the algorithm restarts a new epoch.
Note that the green replay interval (with $m=2$) is also discontinued due to the restart.
}
    \label{fig:illustration}
\end{figure}

Finally, at the end of every replay interval, the algorithm calls the subroutine \testreplay to check whether the data collected in the replay interval and that collected prior to the current block (that is, $\calB_{(i,j-1)}$) are consistent (Line~\ref{line: check restart1}). 
Also, at the end of every block $j$, the algorithm calls another subroutine \testblock to check the consistency between data up to block $j$ and data up to block $k$ for all $k\in \{0, 1, \ldots, j-1\}$ (Line~\ref{line: check restart2}). 
Both tests are in similar spirit to those of~\citep{LuoWA018},
and check the difference of empirical regret or empirical variance of each policy over different sets of data (see Figure~\ref{fig:tests_and_OP}). The difference between the two tests is that they capture non-stationarity at different time steps. If either of the tests indicates that there is a significant distribution change, the algorithm restarts from scratch and enters into the next epoch. 
Also note that if \testblock passes and the algorithm enters into a new block, all unfinished replay intervals will discontinue ($\calS$ is reset to be empty in Line~\ref{line: S reset}).

We provide an illustration of our algorithm in Figure~\ref{fig:illustration}.

\paragraph{Oracle-efficiency.} 
Our algorithm can be implemented efficiently with an ERM oracle.
\citet{AgarwalHsKaLaLiSc14} show that the Optimization Problem (OP) with input $\nu$ can be solved using $\otil(1/\nu)$ oracle calls with a solution that is $\otil(1/\nu)$-sparse. In our case, $\otil(1/\nu)$ is at most $ \otil(\sqrt{KT})$. 
The two tests can also be implemented efficiently by the exact same arguments of~\citep{LuoWA018}.
For example, in \testreplay, to check if there exists a $\pi \in \Pi$ satisfying Eq.~\eqref{eqn:test_cond1}, we can first use two oracle calls to precompute  $\max_{\pi'\in \Pi} \avgR_\calA(\pi')$  and $\max_{\pi'\in \Pi} \avgR_{\calB_{(i,j-1)}}(\pi')$, and collect $\mathcal{T} =\left\{\left(x_t,\frac{-\ips_t}{|\calA|}\right)\right\}_{t\in\calA}\cup\left\{\left(x_t,\frac{4\ips_t}{|\calB_{(i,j-1)}|}\right)\right\}_{t\in\calB_{(i,j-1)}}$. Then we again use an oracle call to find $\max_{\pi\in\Pi}\sum_{(x,r)\in \mathcal{T}}r(\pi(x))$
and add this value to $\max_{\pi'\in \Pi} \avgR_\calA(\pi')-4\max_{\pi'\in \Pi}\avgR_{\calB_{(i,j-1)}}(\pi')$, which is equal to taking the max over $\pi \in \Pi$ of the left hand side of Eq.~\eqref{eqn:test_cond1}. 
It remains to compare this value with the right hand side of Eq.~\eqref{eqn:test_cond1}.
\\




\section{Main Theorem and Proof Outline}

The dynamic regret guarantee of \algo is summarized below:
\begin{theorem}[Main Theorem]
\label{theorem: main theorem}
\algo guarantees with high probability,
\begin{align*}
    \sum_{t=1}^T r_t(\pi_t^*(x_t)) - r_t(a_t) = \otil\left(\min\left\{\sqrt{K(\ln |\Pi|)ST}, \sqrt{K(\ln |\Pi|)T} + (K\ln |\Pi|)^{\frac{1}{3}}\Delta^{\frac{1}{3}}T^{\frac{2}{3}}\right\}\right).
\end{align*}
\end{theorem}

\paragraph{Proof roadmap.}

The rest of the paper proves our main theorem, following these steps:
in Section~\ref{sec: main lemma}, we provide a key lemma that bounds the dynamic regret for any interval within a block (in terms of some algorithm-dependent quantities). In Section~\ref{sec: block reg}, with the help of the key lemma we bound the dynamic regret for a block. In Section~\ref{sec: overall reg} we bound the number of epochs/restarts, and sum up the regret over all blocks in all epochs to get the final bound. 
Since the analysis in Sections~\ref{sec: main lemma} and~\ref{sec: block reg} is all about a fixed epoch $i$, for notation simplicity, we simply write $\calB_{(i,j)}$ and $Q_{(i,j)}$ as $\calB_j$ and $Q_j$ in these two sections.

\subsection{A main Lemma and regret decomposition}\label{sec:interval regret}
\label{sec: main lemma}

To bound the dynamic regret over any interval, we define the concept of {\it excess regret}:

\begin{definition}
\label{lem:excess regret}
For an interval $\calI$ that lies in $[\tau_i + 2^{j-1}L, \tau_i + 2^{j}L-1]$ for some $j>0$, we define its excess regret as
\begin{align*}
    \varepsilon_{\calI} \triangleq \max_{\pi\in\Pi} \Reg_{\calI}(\pi) - 8\whReg_{\calB_{(i,j-1)}}(\pi),
\end{align*}
and its excess regret threshold as $\alpha_\calI = \sqrt{\frac{2KC_0}{|\calI|}} \log_2 T$.
\end{definition}



In words, excess regret of $\calI$ is the maximum discrepancy between a policy's expected static regret on $\calI$ and (8 times) its empirical static regret on the first $j$ blocks.
Large excess regret thus indicates non-stationarity.
We now use the following main lemma to decompose the dynamic regret on $\calI$ based on whether the excess regret reaches the excess regret threshold. 

\begin{lemma}[Main Lemma]\label{lem:interval regret}
With probability $1-\delta$, \algo guarantees for all $j>0$ and any interval $\calI$ that lies in block $j$, 
\begin{align*}
    &\sum_{t\in \calI} \left(r_t(\pi_t^*(x_t)) - r_t(a_t)\right) = \order\left(\left(\sum_{t\in\calI}\sum_{m\in M_t\cup\{j\}}\Kprime\nu_m\right) +  |\calI|\alpha_\calI + |\calI|\Delta_\calI + |\calI|\varepsilon_{\calI}\one\{\varepsilon_\calI > D_3\alpha_\calI\}\right)  
\end{align*}
where $D_3=4.1\times 10^6$. 
\end{lemma}

\begin{proof}
By Azuma's inequality and a union bound over all $T^2$ possible intervals, we have that with probability $1-\delta$, for any interval $\calI$, 
\begin{align}
\sum_{t\in \calI} (r_t(\pi_t^*(x_t)) - r_t(a_t)) 
\leq \sum_{t\in \calI} \E_t[r_t(\pi_t^*(x_t)) - r_t(a_t)]  + \order\left(\sqrt{|\calI| \log(T^2/\delta)}\right), \label{bounding Et Reg step2}
\end{align}
where $\E_t$ is the conditional expectation given everything up to Step 1 of the algorithm of round $t$.
It remains to bound each $\E_t[r_t(\pi_t^*(x_t)) - r_t(a_t)]$.
Depending on the case of replay or non-replay, this term can be written as
\begin{align*}
    \E_t\left[r_t(\pi^*_t(x_t)) - r_t(a_t) \right] 
    &= \begin{cases}
              \displaystyle \sum_{a\in[K]} \sum_{m\in M_t} \frac{Q_{m}^{\nu_m}(a|x_t)}{|M_t|}\E_t\left[r_t(\pi^*_t(x_t)) - r_t(a) \right],  &\text{if $M_t\neq \emptyset$,} \\
               \displaystyle \sum_{a\in[K]} Q_{j}^{\nu_j}(a|x_t) \E_t\left[ r_t(\pi^*_t(x_t)) - r_t(a)\right], &\text{if $M_t= \emptyset$.}
    \end{cases}
\end{align*}    
Now observe that for any $Q$ and $\nu$, by definition of $Q^{\nu}$ we have
\[
\sum_{a\in[K]} Q^{\nu}(a | x_t)\E_t\left[ r_t(\pi^*_t(x_t)) - r_t(a)\right] \leq K\nu + \sum_{\pi\in\Pi} Q(\pi) \Reg_t(\pi).
\]
So we continue to bound $\E_t\left[r_t(\pi^*_t(x_t)) - r_t(a_t) \right]$ by           
\begin{align}
\sum_{m\in M_t\cup \{j\}} K\nu_m + 
         \begin{cases}
         \frac{1}{|M_t|}\sum_{m\in M_t} \sum_{\pi\in\Pi}  Q_{m}(\pi) \Reg_t(\pi), &\text{if $M_t\neq \emptyset$,} \\
         \sum_{\pi\in\Pi}  Q_{j}(\pi) \Reg_t(\pi), &\text{if $M_t= \emptyset$.}
         \end{cases}
\label{bounding Et Reg step1}
\end{align}                  
Next note that for any $t \in \calI$ and $m \in \{1, \ldots, j\}$, we have
\begin{align}
    \sum_{\pi \in \Pi} Q_{m}(\pi) \Reg_t(\pi)  \nonumber 
    &\leq \sum_{\pi \in \Pi} Q_{m}(\pi) \Reg_\calI(\pi) + \order(\Delta_\calI)   \tag{Lemma~\ref{lem:variation and regret}}   \nonumber \\
    &=  \sum_{\pi \in \Pi} 8Q_{m}(\pi)\whReg_{\calB_{j-1}}(\pi) + \order(\Delta_\calI) +  \varepsilon_{\calI} \tag{definition of $\varepsilon_\calI$}   \nonumber \\
    &\leq \sum_{\pi \in \Pi} 8Q_{m}(\pi)\left(4\whReg_{\calB_{m-1}}(\pi)+ D_4\Kprime\nu_{m}\right) + \order(\Delta_\calI) +  \varepsilon_{\calI} \tag{Eq.~\eqref{eqn:test_cond4} does not hold}    \nonumber \\
    &\leq  \order(\Kprime\nu_m + \Delta_\calI) +  \varepsilon_{\calI} \tag{Eq.~\eqref{eqn:op1}}  \nonumber \\
    &\leq  \order(\Kprime\nu_m + \alpha_\calI + \Delta_\calI) +  \varepsilon_{\calI}\mathbf{1}\left\{\varepsilon_{\calI} > D_3\alpha_\calI \right\}. \nonumber  
\end{align}
In fact, the above holds for $m=0$ too since the left hand side is at most $1\leq 4\Kprime\nu_0$.
Combining this inequality with Eq.~\eqref{bounding Et Reg step1} and~\eqref{bounding Et Reg step2}, and noting that the term $\sqrt{|\calI|\log(T^2/\delta)}$ is of order $\mathcal{O}\left(|\calI|\alpha_\calI\right)$ finish the proof.
%
%
\end{proof}


\subsection{Dynamic regret for a block}\label{sec: block reg}
In this section, we bound the dynamic regret of some block $j>0$ within epoch $i$.
This block can be formally written as 
\begin{align}
      \calJ \triangleq  \left[\tau_i, \tau_{i+1}-1\right] \cap \left[\tau_i+2^{j-1}L, \tau_i+2^j L-1\right].    \label{eqn: Js definition}
\end{align}

The idea is to divide $\calJ$ into several intervals, 
apply Lemma~\ref{lem:interval regret} to each of them,
and finally sum up the regret.
Importantly, we need to divide $\calJ$ in a careful way according to the following lemma,
so that the variation on each interval is bounded by its excess regret threshold,
while at the same time the number of intervals is not too large. 
Note that this division only happens in the analysis.

\begin{lemma}
\label{lemma: division}
There is a way to partition any interval $\calJ$ into $\calI_1 \cup \calI_2 \cup \cdots \cup \calI_\Gamma$, such that $\Delta_{\calI_k} \leq \alpha_{\calI_k}, \forall k\in [\Gamma]$, and $\Gamma=\mathcal{O}(\min\{S_\calJ,  (KC_0)^{-\frac{1}{3}}\Delta_\calJ^{\frac{2}{3}}|\calJ|^{\frac{1}{3}} +1 \})$. 
\end{lemma}


For the first $\Gamma-1$ intervals of this partition, we apply Lemma~\ref{lem:interval regret} to each of them. 
Note that the term $|\calI_k|\Delta_{\calI_k}$ in Lemma~\ref{lem:interval regret} can be absorbed by the term $|\calI_k|\alpha_{\calI_k}$ by our partition property. 
%
%
Summing up the bounds from Lemma~\ref{lem:interval regret}, we get the following dynamic regret bound for these $\Gamma-1$ intervals:
\begin{align}
    &\sum_{k=1}^{\Gamma-1} \order\left( \left(\sum_{t\in\calI_k}\sum_{m\in M_t\cup\{j\}}\Kprime \nu_m \right)+  |\calI_k|\alpha_{\calI_k} + |\calI_k|\Delta_{\calI_k} + |\calI_k|\varepsilon_{\calI_k}\one\{\varepsilon_{\calI_k} > D_3\alpha_{\calI_k}\}\right) \nonumber \\
    & \leq\underbrace{ \sum_{k=1}^{\Gamma-1}\sum_{t\in \calI_k} \sum_{m\in M_t\cup\{j\}}\order( \Kprime\nu_m)}_{\term{1}} + \underbrace{\sum_{k=1}^{\Gamma-1} \order(|\calI_k|\alpha_{\calI_k})}_{\term{2}}
    +\underbrace{\sum_{k=1}^{\Gamma-1} \order(|\calI_k|\varepsilon_{\calI_k}\one\{\varepsilon_{\calI_k} > D_3\alpha_{\calI_k}\})}_{\term{3}}. \label{eqn: term decomposition}
\end{align}
For the last interval in the block, it is possible that it was interrupted by a restart, 
which makes the analysis trickier and we defer the details to Appendix~\ref{appendix: first two terms}.  Further bounding \term{1} and \term{2} is relatively straightforward by the definition of $\nu_m$ and $\alpha_{\calI_k}$ and also the construction of $M_t$ (see Appendix~\ref{appendix: bound individual terms}). 
For \term{3}, the idea is that this term is nonzero only when $\varepsilon_{\calI_k}$ is large, which implies that the distribution in $\calI_k$ is quite different from that in $\calB_{j-1}$. In this case we will show that as long as the algorithm starts a replay phase with some ``correct'' index within $\calI_k$, it will detect the non-stationarity with high probability and restart the algorithm. Thus we only need to bound the regret accumulated before this ``correct'' replay phase appears. We provide the complete proof in Appendix~\ref{appendix: detections}, which is the most important part of the analysis. Combining the bounds for these three terms,
we eventually arrive at the following lemma: 
\begin{lemma}\label{lem:block_regret}
With probability $1-\delta$, the following holds for any block $\calJ$ with block index $j>0$: 
\begin{align*}
     \sum_{t\in \calJ}(r_t(\pi_t^*)-r_t(a_t)) =  \otil\left(\min \left\{\sqrt{KC_0S_{\calJ}2^j L},  \sqrt{KC_0 2^j L} + (KC_0)^{\frac{1}{3}}\Delta_{\calJ}^{\frac{1}{3}}(2^j L)^{\frac{2}{3}}  \right\}\right). 
\end{align*}
\end{lemma}
Note that $2^jL$ is the length of block $\calJ$ unless there is a restart triggered within this block, in which case the length is smaller.

\subsection{Combining regret over blocks and epochs}\label{sec: overall reg}
We finally sum up the dynamic regret over blocks and epochs. To this end, we reintroduce the subscripts $i,j$ in our notations, and write epoch $i$ as
$
  \calE_i = [\tau_i, \tau_{i+1}-1]
$
and block $j$ (for $j>0$) in epoch $i$ as 
$
   \calJ_{ij} = [\tau_i + 2^{j-1}L, \tau_i + 2^jL-1] \cap \calE_i. 
$

\paragraph{Dynamic regret for an epoch.}
The last block index in epoch $i$ is $\max\{0,\lceil \log_2 (|\calE_i|/L) \rceil\}$, which we denote by $j^*$. Using Lemma~\ref{lem:block_regret}, we combine the regret over all blocks in epoch $i$ and upper bound the regret of epoch $i$ simultaneously by (using the bound in terms of number of switches)
\begin{align*}
   & \otil\left(L +  \sum_{j=1}^{j^*} \sqrt{KC_0 S_{\calJ_{ij}}2^j L}  \right) 
= \otil\left(KC_0  + \sqrt{KC_0 \sum_{j=1}^{j^*}S_{\calJ_{ij}}\sum_{j=1}^{j^*} 2^jL} \right)  \tag{Cauchy-Schwarz} \\ 
&= \otil\left(KC_0 + \sqrt{KC_0 (S_{\calE_i}+j^*)|\calE_i|} \right) 
 = \otil\left( \sqrt{KC_0 S_{\calE_i}|\calE_i|} \right)
\end{align*}
and similarly by (using the bound in terms of variation and H\"{o}lder inequality)
\begin{align*}
\otil \left(L +  \sum_{j=1}^{j^*} \sqrt{KC_0 2^j L} +  \sum_{j=1}^{j^*} (KC_0)^{\frac{1}{3}}\Delta_{\calJ}^{\frac{1}{3}}(2^j L)^{\frac{2}{3}} \right) 
= \otil \left(\sqrt{KC_0 |\calE_i|}+ (KC_0)^{\frac{1}{3}}\Delta_{\calE_i}^{\frac{1}{3}}|\calE_i|^{\frac{2}{3}} \right).
\end{align*}

\paragraph{Combining regret over epochs. } 
For the last step of combining all epochs, we make use the following lemma which bounds the number of epochs (see Appendix~\ref{appendix:restart} for the proof). 
\begin{lemma}
\label{lemma: number of restarts}
Denote the total number of epochs by $E$. 
With probability at least $1-\delta/2$, 
we have $E\leq \min\{S, (KC_0)^{-\frac{1}{3}}\Delta^{\frac{2}{3}}T^{\frac{1}{3}}+1\}$.
\end{lemma}

Therefore, summing up the previous bounds over all epochs,
we arrive at the final dynamic regret bound, which is the minimum of the following two:
\begin{align*}
    \otil\left(\sum_{i=1}^{E}\sqrt{KC_0 S_{\calE_i}|\calE_i|} \right) 
    &\leq \otil\left(\sqrt{KC_0 \left(\sum_{i=1}^{E}S_{\calE_i}\right)\left(\sum_{i=1}^{E}|\calE_i|\right)} \right) \\
    &=\otil\left(\sqrt{KC_0(S+E)T} \right)
    = \otil\left(\sqrt{KC_0ST} \right), 
\end{align*}
and by 
\begin{align*}
    \otil\left(\sum_{i=1}^{E} \left(\sqrt{KC_0|\calE_i|} + (KC_0)^{\frac{1}{3}}\Delta_{\calE_i}^{\frac{1}{3}}|\calE_i|^{\frac{2}{3}} \right) \right)
    &\leq \otil\left(\sqrt{KC_0 ET} + (KC_0)^{\frac{1}{3}} \left(\sum_{i=1}^{E}\Delta_{\calE_i}\right)^{\frac{1}{3}}T^{\frac{2}{3}}\right) \\
    & = \otil\left(\sqrt{KC_0T} + (KC_0)^{\frac{1}{3}}\Delta^{\frac{1}{3}}T^{\frac{2}{3}}\right). 
\end{align*}
This proves the bound stated in the main theorem.

\acks{The authors would like to thank Peter Auer for the discussion about the possibility of getting optimal bounds for our problem, and to thank Peter Auer, Pratik Gajane, Ronald Ortner for kindly sharing their manuscript of \citep{AuerGO018} before it was public. HL and CYW are supported by NSF Grant \#1755781.}

\bibliography{ref}

\appendix

\section{Useful Lemmas}\label{app:small_lemmas}
In this section we prove two small lemmas that are useful for our analysis.

\subsection{Discrepancy between intervals}
\label{appendix: variations inequality}

The following results allow us to relate regret and variance measured on one interval to those measured on another, with the price in terms of the distribution variation.
\begin{lemma}\label{lem:variation and regret}
For any interval $\calI$, its sub-intervals $\calI_1, \calI_2\subseteq \calI$, and any $\pi\in\Pi$, we have
\begin{align*}
\bigabs{\Reg_{\calI_1}(\pi)-\Reg_{\calI_2}(\pi)} \leq 2\Delta_\calI.
\end{align*}
\end{lemma}
\begin{proof}
Let $\pi_{\calI_1}^\star=\argmax_{\pi\in\Pi}\calR_{\calI_1}(\pi)$ and $\pi_{\calI_2}^\star=\argmax_{\pi\in\Pi}\calR_{\calI_2}(\pi)$. Then 
\begin{align*}
\Reg_{\calI_1}(\pi)-\Reg_{\calI_2}(\pi)=\calR_{\calI_1}(\pi_{\calI_1}^\star)-\calR_{\calI_1}(\pi)-\calR_{\calI_2}(\pi_{\calI_2}^\star)+\calR_{\calI_2}(\pi).
\end{align*}
By definition we have
\begin{align*}
-\var_\calI\leq \calR_{\calI_2}(\pi)-\calR_{\calI_1}(\pi)\leq \var_\calI , 
\end{align*}
and 
\begin{align*}
-\var_\calI\leq \calR_{\calI_1}(\pi_{\calI_2}^\star)-\calR_{\calI_2}(\pi_{\calI_2}^*) \leq \calR_{\calI_1}(\pi_{\calI_1}^\star)-\calR_{\calI_2}(\pi_{\calI_2}^\star)\leq \calR_{\calI_1}(\pi_{\calI_1}^\star)-\calR_{\calI_2}(\pi_{\calI_1}^\star)\leq \var_\calI. 
\end{align*}
Combining them we get the desired bound.
\end{proof}

\begin{lemma}\label{lem:variation and TVD}
For any interval $\calI$, its sub-intervals $\calI_1, \calI_2\subseteq \calI$, any $Q \in \Delta^\Pi$, $\mu \in (0,1/K]$, and $\pi\in\Pi$, we have
\begin{align*}
\bigabs{V_{\calI_1}(Q,\mu, \pi)-V_{\calI_2}(Q,\mu, \pi)} \leq \frac{\Delta_\calI}{\mu}.
\end{align*}
\end{lemma}
\begin{proof}
For any $s,t\in\calI$ (assuming $s<t$), any $P$, and $\pi\in\Pi$, 
\begin{align*}
\bigabs{V_s(Q,\mu,\pi)-V_t(Q,\mu,\pi)} 
& = \Bigg\vert\E_{\calD_s^\calX} \left[ \frac{1}{Q^\mu(\pi(x)|x)}\right] - \E_{\calD_t^\calX}\left[ \frac{1}{Q^\mu(\pi(x)|x)}\right]\Bigg\vert \\
& = \Bigg\vert \int_{\calX} (\calD_s^\calX(x)-\calD_t^\calX(x)) \frac{1}{Q^\mu(\pi(x)|x)}dx \Bigg\vert \\
& \leq \frac{1}{\mu} \int_\calX \bigabs{\calD_s^\calX(x)-\calD_t^\calX(x)}dx \\
& \leq \frac{1}{\mu} \sum_{\tau=s+1}^t \TVD{\calD_\tau - \calD_{\tau-1}} \leq \frac{\Delta_\calI}{\mu}. 
\end{align*}
Therefore, 
\begin{align*}
\bigabs{V_{\calI_1}(Q,\mu,\pi) - V_{\calI_2}(Q,\mu,\pi)} \leq \frac{1}{\abs{\calI_1}}\frac{1}{\abs{\calI_2}} \sum_{s\in\calI_1}\sum_{t\in\calI_2}\bigabs{V_s(Q,\mu,\pi)-V_t(Q,\mu,\pi)}\leq \frac{\Delta_\calI}{\mu}. 
\end{align*}
\end{proof}

\subsection{Partitioning an interval}
\label{sec:partition}

We prove Lemma~\ref{lemma: division} in this section,
which states that for any interval $\calJ$,
there exists a way to partition $\calJ$ into $\calI_1 \cup \calI_2 \cup \cdots \cup \calI_\Gamma$, such that 
\[
\Delta_{\calI_k} \leq \alpha_{\calI_k} = \sqrt{\frac{2KC_0}{|\calI_k|}} \log_2 T, \;\;\forall k\in [\Gamma],
\] 
and 
\[
\Gamma=\mathcal{O}\left(\min\left\{S_\calJ,  (KC_0)^{-\frac{1}{3}}\Delta_\calJ^{\frac{2}{3}}|\calJ|^{\frac{1}{3}} +1 \right\}\right). 
\]
We prove this by giving an explicit greedy ``algorithm'',
but we emphasize that this only happens in the analysis and is never really needed to be executed. \\

\begin{proof}[of Lemma~\ref{lemma: division}]
Consider the following partitioning procedure.

\begin{algorithm}[H]
\DontPrintSemicolon
\caption{Partitioning an interval}
\label{alg:partition}
{\bf Input}: an interval $\calJ=[s, e]$. \\
{\bf Initialize}: Let $k=1$, $s_1=s,t=s$.  \\
\While{$t\leq e$}{
   \If{$\Delta_{[s_k,t]} \leq \sqrt{\frac{KC_0}{t-s_k+1}}$ and $\Delta_{[s_k,t+1]} > \sqrt{\frac{KC_0}{t-s_k+2}}$}{
      Let $e_k\leftarrow t$, $\calI_k\leftarrow [s_k,e_k]$, \\
      $ k\leftarrow k+1$, $s_k\leftarrow t+1$.  
   }
   $t\leftarrow t+1$.
}
\If{$s_k\leq e$}{
$e_k\leftarrow e, \calI_{k}\leftarrow [s_k,e_k]$.
}
\end{algorithm}

It is clear that this procedure ensures $\Delta_{\calI_k} \leq \alpha_{\calI_k}$ for all $k$.
It remains to bound $\Gamma$.
If $\Gamma>1$, by the procedure and the decomposability of variation we have
\begin{align*}
\Delta_{[s, e]} 
&\geq \Delta_{[s_1, e_1+1]}+\Delta_{[s_2, e_2+1]}+\ldots+\Delta_{[s_{\Gamma-1}, e_{\Gamma-1}+1]} \geq \sum_{k=1}^{\Gamma-1} \sqrt{\frac{KC_0}{e_k-s_k+2}}=\sum_{k=1}^{\Gamma-1} \sqrt{\frac{KC_0}{\abs{\calI_k}+1}}.  
\end{align*}
On the other hand H\"{o}lder's inequality implies
\begin{align*}
\left( \sum_{k=1}^{\Gamma-1} \sqrt{\frac{KC_0}{\abs{\calI_k}+1}} \right)^{\frac{2}{3}} \left(\sum_{k=1}^{\Gamma-1} (|\calI_k|+1) \right)^{\frac{1}{3}}\geq (\Gamma-1)(KC_0)^{\frac{1}{3}}. 
\end{align*}
Combining the two inequalities above, we get 
\begin{align*}
\Gamma -1 \leq (KC_0)^{-\frac{1}{3}}  \left(\sum_{k=1}^{\Gamma-1} (|\calI_k|+1) \right)^{\frac{1}{3}} \Delta_{[s,e]}^{\frac{2}{3}}\leq \order\left((KC_0)^{-\frac{1}{3}}|\calJ|^{\frac{1}{3}} \Delta_{[s,e]}^{\frac{2}{3}}\right).  
\end{align*}
It is also clear from the condition $\Delta_{[s_k, t+1]}>\sqrt{\frac{KC_0}{t-s_k+2}}$  that the procedure creates one interval only when the distribution switches. Therefore, 
$
\Gamma-1\leq S_{[s,e]}-1,
$
which completes the proof.
\end{proof}

\section{Concentration Results}
\label{appendix:concentration}

This section is dedicated to all concentration results we need for our analysis.
First we introduce some notations and technical lemmas.

\begin{definition}
Define $U_t(\pi)$ as the conditional variance of the reward estimation for policy $\pi$ at time $t$ (given everything before time $t$), that is, 
\begin{align*}
    U_t(\pi) = 
    \E_{t}\left[(\ips_t(\pi(x_t)) - \calR_t(\pi))^2\right].
\end{align*}
\end{definition}
%
%
Also recall the variance notation defined in Section~\ref{sec:setup}:
\[
\whV_\calI (Q, \nu, \pi) = \frac{1}{|\calI|} \sum_{t \in \calI} \left[ \frac{1}{Q^{\nu}(\pi(x_t) | x_t)} \right], 
\quad V_\calI(Q, \nu, \pi) = \frac{1}{|\calI|} \sum_{t \in \calI} \E_{x \sim \calD_t^\calX} \left[ \frac{1}{Q^{\nu}(\pi(x) | x)} \right]
\]
($\whV_t$ and $V_t$ are shorthands for $\whV_{[t, t]}$ and $V_{[t, t]}$ respectively).
The following lemma connects these notions of variance for our algorithm.
\begin{lemma}
\label{lemma:variance_upperbound}
For any policy $\pi$ and any time $t$ in epoch $i$ and block $j$,
if $M_t\neq \emptyset$, then $U_t(\pi)\leq V_t(Q_{(i,m)}, \nu_m, \pi)\log_2 T$ for any $m\in M_t$; if $M_t=\emptyset$ , then $U_t(\pi)\leq V_t(Q_{(i,j)}, \nu_j, \pi)$.
\end{lemma}
\begin{proof}
When $M_t\neq \emptyset$, the distribution over actions played by the algorithm is
\begin{equation*}
    p_{t}(\cdot)=\frac{1}{|M_t|}\sum_{m\in M_t}Q_{(i,m)}^{\nu_m}(\cdot|x_t).
\end{equation*}
Thus the variance is bounded as
\begin{align*}
    U_t(\pi)
    &\leq \E\left[\ips_t(\pi(x_t))^2\right]
    = \mathbb{E}_{(x,r)\sim \mathcal{D}_t} \left[ p_t(\pi(x))\cdot \frac{r_t(\pi(x))^2}{p_t(\pi(x))^2} \right]\\
    &\leq \E_{x\sim \calD_t^\calX} \left[ \frac{|M_t|}{\sum_{m'\in M_t}Q_{(i,m')}^{\nu_{m'}}(\pi(x)|x)} \right] \le \mathbb{E}_{x\sim \mathcal{D}_t^\calX}\left[\frac{\log_2T}{Q_{(i,m)}^{\nu_m}(\pi(x)|x)}\right]=V_t(Q_{(i,m)}, \nu_m, \pi)\log_2T,
\end{align*}
where we use the fact that $|M_t|\leq j\leq \log_2\frac{T}{L}+1 \leq \log_2 T$. 
Similarly, when $M_t=\emptyset$, we have
\begin{align*}
    U_t(\pi)\leq \E_{x\sim \calD_t^\calX}\left[\frac{1}{Q_{(i,j)}^{\nu_j}(\pi(x)|x)}\right]=V_t(Q_{(i,j)}, \nu_j, \pi). 
\end{align*}
\end{proof}


We repeatedly make use of the following standard concentration bounds for martingales (which is a version of the Freedman's inequality; see for example~\citep{BeygelzimerLaLiReSc11}).
\begin{lemma}[Freedman's inequality]
\label{lem:freedman}
Let $X_1, \ldots, X_n \in \fR$ be a martingale difference sequence 
with respect to some filtration $\calF_0, \calF_1, \ldots$.
Assume $X_i \leq R$ a.s. for all $i$.
Then for any $\delta \in (0,1)$ and $\lambda \in [0, 1/R]$, with probability at least $1-\delta$, we have
\[
\sum_{i=1}^n X_i \leq \lambda V + \frac{\ln(1/\delta)}{\lambda}
\]
where $V = \sum_{i=1}^n \E[X_i^2 | \calF_{i-1}]$.
\end{lemma}

\subsection{Concentration of reward estimator and variance}

The following concentration results on the reward estimator and variance are based on applications of Lemma~\ref{lem:freedman}.
Recall $C_0$ is defined in the algorithm.

\begin{lemma}\label{lem:variance_deviation2}
With probability at least $1 - \delta/4$, for all $Q\in \Delta^{\Pi}$,  all $\nu\in \{\nu_0, \nu_1, \ldots, \nu_{j_{\max}}\}$, where $j_{\max}\triangleq \lceil\log_2 T\rceil$, all $\pi\in\Pi$, and all intervals $\calI$, it holds that
\begin{equation}\label{eq:variance_deviation2}
V_{\calI}(Q, \nu, \pi) \leq 6.4 \whV_\calI(Q, \nu, \pi) + \frac{80C_0}{\nu^2|\calI|}, \qquad  \whV_{\calI}(Q, \nu, \pi) \leq 6.4 V_\calI(Q, \nu, \pi) + \frac{80C_0}{\nu^2|\calI|}.
\end{equation}
\end{lemma}
\begin{proof}
This is a consequence of the contexts being drawn independently, and is not related to the algorithm. Therefore, we can apply the same argument of~\citep[Theorem 6]{DudikHsKaKaLaReZh11}, \citep[Lemma 10]{AgarwalHsKaLaLiSc14}, or~\citep[Lemma 15]{LuoWA018}. For example, as shown by \citep[Lemma 10]{AgarwalHsKaLaLiSc14}, with probability $1-\delta/(4T)$, for all $Q$, all $\pi$, and all $\calI$, $V_{\calI}(Q, \nu, \pi)-6.4 \whV_\calI(Q, \nu, \pi)$ and $\whV_{\calI}(Q, \nu, \pi)-6.4 V_\calI(Q, \nu, \pi)$ are both upper bounded by 
\begin{align*}
\frac{75\ln (|\Pi|)}{\nu^2|\calI|}+\frac{6.3\ln (8T^3|\Pi|^2/\delta)}{\nu |\calI|} \leq \frac{80 \ln (8T^3|\Pi|^2/\delta)}{\nu^2|\calI|}. 
\end{align*}
Another union bound over $\nu$ finishes the proof. 
\end{proof}

\begin{lemma}\label{lem:reward_deviation2}
With probability at least $1-\delta/4$, for all policy $\pi \in \Pi$, 
we have for all interval $\calB_j$ that corresponds to the first $j+1$ blocks of some epoch,
\begin{equation*}
\left|\avgR_{\calB_j}(\pi) - \calR_{\calB_j}(\pi)\right| \leq \frac{\nu_j}{|\calB_j|\log_2 T} \sum_{t\in\calB_j} U_t(\pi) + \frac{C_0 \log_2 T}{\nu_j |\calB_j|},
\end{equation*}
and for all interval $\calA$ that is covered by some replay phase of index $m$,
\begin{equation*}
\left|\avgR_{\calA}(\pi) - \calR_{\calA}(\pi)\right| \leq \frac{\nu_m}{|\calA|\log_2 T} \sum_{t\in\calA} U_t(\pi) + \frac{C_0 \log_2 T}{\nu_m |\calA|}.
\end{equation*}
\end{lemma}

\begin{proof}
We simply apply Lemma~\ref{lem:freedman} to the sequence $(\ips_t(\pi(x_t)) - \calR_t(\pi))\one\{t \in \calB_j\}$ for every interval and every $\pi$ (with a union bound).
Note that these random variables are bounded by $1/\nu_j$ so we can pick $\lambda=\nu_j/\log_2 T$.
Similarly for the second statement we apply Lemma~\ref{lem:freedman} to the sequence$(\ips_t(\pi(x_t)) - \calR_t(\pi))\one\{m \in M_t\}$ with $\lambda = \nu_m/\log_2 T$.
%
\end{proof}

Since most analysis conditions on these concentration results, we denote the event formally below, which clearly happens with probability at least $1-\delta/2$.
\begin{definition}[\event{1}]
Define \event{1} as the event that all bounds described in Lemma~\ref{lem:reward_deviation2} and Lemma~\ref{lem:variance_deviation2} hold. 
\end{definition}

\subsection{Concentration of regret}

In this section we prove three main concentration results on regret,
which play a crucial role later in our analysis.
We focus on a specific epoch $i$ and for simplicity use $\calB_j$ and $Q_j$ as shorthands for $\calB_{(i,j)}$ and $Q_{(i,j)}$ respectively
(we remind the reader that these notations are defined in the algorithm).

\begin{lemma} \label{lem:Reg gap for B_j}
Assume \event{1} holds, and assume that there is no restart triggered in $\calB_j$, then the following hold for all $\pi\in\Pi$: 
\begin{align*}
    \Reg_{\calB_{j}}(\pi) \leq 2\whReg_{\calB_{j}}(\pi) + C_1\Kprime\nu_j + C_2\Delta_{\calB_j},\\
    \whReg_{\calB_{j}}(\pi) \leq 2\Reg_{\calB_{j}}(\pi) + C_1\Kprime\nu_j + C_2\Delta_{\calB_j},  
\end{align*}
where $C_1=2000, C_2=24$. 
\end{lemma}
\begin{lemma}
\label{lem:2 and 3 hold}
Assume \event{1} holds. Let $\calA$ be a complete replay phase of index $m$ (that is, $|\calA|=2^mL$).
If for all $\pi$, Eq.~\eqref{eqn:test_cond2} and Eq.~\eqref{eqn:test_cond3} in \testreplay do not hold, then the following hold for all $\pi$:
\begin{align*}
    \Reg_\calA(\pi) \leq 2\whReg_\calA(\pi) + C_3\Kprime\nu_m + C_4\Delta_\calA,\\
    \whReg_\calA(\pi) \leq 2\Reg_\calA(\pi) + C_3\Kprime\nu_m + C_4\Delta_\calA,  
\end{align*}
where $C_3=2\times 10^6, C_4=24$. 
\end{lemma}
\begin{lemma}
\label{lem: regret relation with [1,e]}
Assume \event{1} holds. Let $\calA=[s,e]$ be a complete replay phase of index $m$ (thus $|\calA|=2^mL$). 
Then the following hold for all $\pi$:
\begin{align*}
    \Reg_\calA(\pi) \leq 2\whReg_\calA(\pi) + C_5\Kprime\nu_m + C_6\Delta_{[\tau_i,e]},\\
    \whReg_\calA(\pi) \leq 2\Reg_\calA(\pi) + C_5\Kprime\nu_m + C_6\Delta_{[\tau_i,e]},
\end{align*}
where $C_5=2000, C_6=24$.
\end{lemma}
To prove these results, we first prove the following auxiliary lemma. Basically it shows that in an interval $\calI$, if we can bound the instant variance of a policy $\pi$ by some quantity that is proportional to the regret of $\pi$, then $\Reg_\calI(\pi)$ and $\whReg_\calI(\pi)$ are close. 
\begin{lemma}
\label{lem:useful}
Assume \event{1} holds. 
Consider an interval $\calI$ that is either $\calB_j$ or $\calA$ as defined in Lemma~\ref{lem:reward_deviation2},
and let $\mu$ be $\nu_j/\log_2 T$ if $\calI$ is $\calB_j$ and $\nu_m/\log_2 T$ if $\calI$ is $\calA$.
If either of the following conditions holds: 
\begin{align*}
    U_t(\pi) \leq \frac{\Reg_\calI(\pi)}{3\mu} + Z, \ \ \ \ \ \ \ \forall t\in\calI, \forall \pi\in\Pi,\\
    U_t(\pi) \leq \frac{\whReg_\calI(\pi)}{3\mu} + Z, \ \ \ \ \ \ \ \forall t\in\calI, \forall \pi\in\Pi,
\end{align*} 
for some $Z > 0$, then we have
\begin{align*}
    \Reg_\calI(\pi) \leq 2\whReg_\calI(\pi) + 8\mu Z + \frac{8C_0}{\mu|\calI|}, \ \ \ \ \  \whReg_\calI(\pi) \leq 2\Reg_\calI(\pi) + 8\mu Z + \frac{8C_0}{\mu|\calI|}. 
\end{align*}
\end{lemma}
\begin{proof}{\textbf{of Lemma~\ref{lem:useful}.}}
Suppose we have $ U_t(\pi) \leq \frac{\Reg_\calI(\pi)}{3\mu} + Z$ for all $t\in\calI$ and $\pi\in\Pi$, then
    \begin{align*}
       &\Reg_{\calI}(\pi) - \whReg_{\calI}(\pi)\\
       &= \calR_{\calI}(\pi^*_\calI)  - \calR_{\calI}(\pi) - \avgR_{\calI}(\whpi_\calI) + \avgR_\calI(\pi)\\
       &\leq \left( \calR_{\calI}(\pi_{\calI}^*)- \avgR_{\calI}(\pi^*_{\calI}) \right) + \left(\avgR_{\calI}(\pi) - \calR_{\calI}(\pi)\right) \tag{optimality of $\whpi_{\calI}$}\\
       &\leq \frac{\mu}{|\calI|} \sum_{t\in \calI}\left( U_t(\pi^*_{\calI}) +  U_t(\pi)\right) + \frac{2C_0}{\mu|\calI|} \tag{\event{1}}\\
       &\leq  \frac{1}{3} \Reg_\calI(\pi_\calI^*) + \frac{1}{3} \Reg_\calI(\pi) + 2\mu Z + \frac{2C_0}{\mu|\calI|}, \\
       &=  \frac{1}{3} \Reg_\calI(\pi) + 2\mu Z + \frac{2C_0}{\mu|\calI|},  \tag{$\Reg_\calI(\pi_\calI^*)=0$}
   \end{align*}
   which gives $\Reg_{\calI}(\pi)\leq \frac{3}{2}\whReg_{\calI}(\pi) + 3\mu Z + \frac{3C_0}{\mu|\calI|}\leq 2\whReg_{\calI}(\pi) + 8\mu Z + \frac{8C_0}{\mu|\calI|}$, proving the first part of the lemma. 
   On the other hand, 
   \begin{align*}
       &\whReg_{\calI}(\pi) - \Reg_{\calI}(\pi)\\
       &= \avgR_{\calI}(\whpi_\calI) - \avgR_\calI(\pi)-\calR_{\calI}(\pi^*_\calI)  + \calR_{\calI}(\pi) \\
       &\leq \left( \avgR_{\calI}(\whpi_{\calI}) -\calR_{\calI}(\whpi_{\calI}) \right) + \left(\calR_{\calI}(\pi)-\avgR_{\calI}(\pi) \right) \tag{optimality of $\pi^*_\calI$}\\
       &\leq \frac{\mu}{|\calI|} \sum_{t\in \calI}\left(U_t(\whpi_{\calI}) +  U_t(\pi)\right) + \frac{2C_0}{\mu|\calI|} \tag{\event{1}}\\
       &\leq  \frac{1}{3} \Reg_\calI(\pi) + \frac{1}{3} \Reg_\calI(\whpi_{\calI}) + 2\mu Z + \frac{2C_0}{\mu|\calI|}\\
       &\leq  \frac{1}{2} \whReg_\calI(\pi) + \frac{1}{2}\whReg_\calI(\whpi_\calI) + 4\mu Z + \frac{4C_0}{\mu|\calI|}  \\
       & = \frac{1}{2} \whReg_\calI(\pi)  + 4\mu Z + \frac{4C_0}{\mu|\calI|}.    \tag{$\whReg_\calI(\whpi_\calI)=0$}
   \end{align*}
    where in the second to last inequality we use the fact $\Reg_{\calI}(\pi)\leq \frac{3}{2}\whReg_{\calI}(\pi) + 3\mu Z + \frac{3C_0}{\mu|\calI|}$ for all $\pi$, which we just obtained previously. The last inequality gives $
        \whReg_{\calI}(\pi)\leq 2\Reg_{\calI}(\pi) + 8\mu Z + \frac{8C_0}{\mu |\calI|}$, proving the second part. 
        
        The proof under the second condition proceeds in the exact same way.
\end{proof}
Now we are ready to prove the three lemmas. 
We will frequently use the following facts: 
\begin{align*}
    \frac{C_0}{\nu_j^2 |\calB_j|} = \frac{C_0}{\nu_j^2 2^j L} = K,  \ \ \ \nu_0 = \sqrt{\frac{C_0}{K\lceil 4K C_0\rceil}} \in \left[\frac{1}{4K}, \frac{1}{2K}\right]. 
\end{align*}

\begin{proof}{\textbf{of Lemma~\ref{lem:Reg gap for B_j}.}}
    Assume \event{1} holds.  We prove by induction on $j$. When $j=0$, we have $ \Reg_{\calB_{0}}(\pi)\leq 1 \leq 4\Kprime\nu_0$, and
    \begin{align*}
        \whReg_{\calB_0}(\pi)-\Reg_{\calB_0}(\pi) 
        &= \avgR_{\calB_0}(\whpi_{\calB_0}) - \avgR_{\calB_0}(\pi) - \calR_{\calB_0}(\pi^*_{\calB_0}) + \calR_{\calB_0}(\pi)\\
        &\leq \avgR_{\calB_0}(\whpi_{\calB_0}) - \calR_{\calB_0}(\whpi_{\calB_0}) - \avgR_{\calB_0}(\pi)  + \calR_{\calB_0}(\pi) \tag{by the optimality of $\pi_{\calB_0}^*$}\\
        &\leq 2\left( \frac{\nu_0}{|\calB_0|}\sum_{t\in \calB_0}\frac{1}{\nu_0} + \frac{C_0}{\nu_0 L} \right)\leq 4, \tag{\event{1}}
    \end{align*}
    and thus $\whReg_{\calB_0}(\pi)\leq 5\leq 20\Kprime\nu_0$. 
    Below we prove the inequalities for a general $j$, assuming that they hold for $\{0,\ldots, j-1\}$.  
    For all $t\in \calB_{j}$, and all $m\in [1,j]$, 
    \begin{align}
        V_t(Q_m, \nu_m, \pi) &\leq V_{\calB_{m-1}}(Q_m, \nu_m, \pi) + \frac{\Delta_{\calB_j}}{\nu_m} \tag{Lemma~\ref{lem:variation and TVD}}\\
        &\leq 6.4\whV_{\calB_{m-1}}(Q_m, \nu_m, \pi) + \frac{80C_0}{\nu_m^2 |\calB_{m-1}|} + \frac{\Delta_{\calB_j}}{\nu_m} \tag{\event{1}} \\
        &\leq 6.4\left(2K + \frac{\whReg_{\calB_{m-1}}(\pi)}{C\nu_m}\right) + \frac{80C_0}{\nu_m^2 2^{m-1}L} + \frac{\Delta_{\calB_j}}{\nu_m} \tag{Eq.~\eqref{eqn:op2}}\\
        &\leq 6.4\left(2K + \frac{2\Reg_{\calB_{m-1}}(\pi) + C_1\Kprime \nu_{m-1}+C_2\Delta_{\calB_j}}{C\nu_m}\right) + 160K + \frac{\Delta_{\calB_j}}{\nu_m}\tag{By induction hypothesis} \\
        &\leq \frac{\Reg_{\calB_{m-1}}(\pi)}{3\nu_m} +C_7\Kprime + \frac{2\Delta_{\calB_j}}{\nu_m} \tag{let $C_7=\frac{\sqrt{2}C_1}{C}+12.8+160$} \\
        &\leq \frac{\Reg_{\calB_j}(\pi)}{3\nu_m} + C_7\Kprime + \frac{3\Delta_{\calB_j}}{\nu_m} \tag{Lemma~\ref{lem:variation and regret}} \\
        &\leq \frac{\Reg_{\calB_j}(\pi)}{3\nu_j} + C_7\Kprime + \frac{3\Delta_{\calB_j}}{\nu_j}.  \tag{$\nu_m\geq \nu_j$}
        & \\ \label{eqn:V bound by Reg}
    \end{align}
    Besides, when $j=0$, $V_t(Q_0, \nu_0, \pi)\leq \frac{1}{\nu_0}\leq 4\Kprime$. By Lemma~\ref{lemma:variance_upperbound}, we always have $U_t(\pi)\leq V_t(Q_m, \nu_m, \pi)\log_2 T$ for some $m\in[0,j]$. Therefore, $U_t(\pi)\leq \left(\frac{\Reg_{\calB_j}(\pi)}{3\nu_j} + C_7\Kprime + \frac{3\Delta_{\calB_j}}{\nu_j}\right)\log_2 T$.  
    Using Lemma~\ref{lem:useful} with $ Z=\left(C_7\Kprime + \frac{3\Delta_{B_j}}{\nu_j}\right)\log_2 T $, we get the two desired inequalities. 
\end{proof}

\begin{proof}{\textbf{of Lemma~\ref{lem:2 and 3 hold}.}}
    For all $t\in\calA$, 
    \begin{align*}
        V_t(Q_m, \nu_m, \pi)
        &\leq V_\calA(Q_m, \nu_m, \pi) + \frac{\Delta_\calA}{\nu_m} \tag{Lemma~\ref{lem:variation and TVD}}\\
        &\leq 6.4\whV_\calA(Q_m, \nu_m, \pi) + \frac{80C_0}{\nu_m^2 |\calA|} + \frac{\Delta_\calA}{\nu_m} \tag{\event{1}}\\
        &\leq 6.4\left( 41\hat{V}_{\calB_{j-1}}(Q_m, \nu_m, \pi) + D_2K \right) + 80K + \frac{\Delta_\calA}{\nu_m} \tag{Eq.~\eqref{eqn:test_cond3} does not hold}\\
        &\leq  263\whV_{\calB_{j-1}}(Q_m, \nu_m, \pi) + (6.4D_2+80)K+  \frac{\Delta_\calA}{\nu_m}  \\ 
        &\leq 263(41\whV_{\calB_{m-1}}(Q_m, \nu_m, \pi) + D_5K) + (6.4D_2+80)K+  \frac{\Delta_\calA}{\nu_m} \tag{Eq.~\eqref{eqn:test_cond6} does not hold}\\
        &\leq \frac{\whReg_{\calB_{m-1}}(\pi)}{1000\nu_m} + C_8K+  \frac{\Delta_\calA}{\nu_m}  \tag{by Eq.~\eqref{eqn:op2}}\\
        & \tag{let $C_8=  263\times 41\times 2 +263\times D_5 + 6.4D_2+80$} \\
        &\leq \frac{4\whReg_{\calB_{j-1}}(\pi) + D_4\Kprime \nu_m}{1000\nu_m} + C_8K+  \frac{\Delta_\calA}{\nu_m}  \tag{Eq.~\eqref{eqn:test_cond5}} \\
        &= \frac{\whReg_{\calB_{j-1}}(\pi)}{250\nu_m} + \left( C_8+0.001D_4 \right)\Kprime + \frac{\Delta_\calA}{\nu_m} \\
        &\leq \frac{4\whReg_{\calA}(\pi)+D_1\Kprime\nu_m}{250\nu_m} +(C_8+0.001D_4)\Kprime +  \frac{\Delta_\calA}{\nu_m} \tag{Eq.~\eqref{eqn:test_cond2} does not hold} \\
        &\leq \frac{\whReg_{\calA}(\pi)}{3\nu_m} + C_9\Kprime + \frac{\Delta_\calA}{\nu_m}.   \tag{let $C_9=\frac{D_1}{250}+C_8+0.001D_4$}
    \end{align*}
    Using $U_t(\pi)\leq V_t(Q_m, \nu_m, \pi)\log_2 T$ (Lemma~\ref{lemma:variance_upperbound}) and invoking Lemma~\ref{lem:useful} with \[Z=\left(C_9\Kprime + \frac{\Delta_{\calA}}{\nu_m}\right)\log_2 T\] finish the proof. 
   
\end{proof}

\begin{proof}{\textbf{of Lemma~\ref{lem: regret relation with [1,e]}.}}
For all $t\in \calA$, 
\begin{align*}
    V_t(Q_m, \nu_m, \pi) 
    &\leq V_{\calB_{m-1}}(Q_m, \nu_m, \pi) + \frac{\Delta_{[\tau_i,e]}}{\nu_m} \tag{Lemma~\ref{lem:variation and TVD}}\\
    &\leq 6.4\whV_{\calB_{m-1}}(Q_m, \nu_m, \pi) + \frac{80C_0}{\nu_m^2 |\calB_{m-1}|} + \frac{\Delta_{[\tau_i,e]}}{\nu_m}  \tag{\event{1}}\\
    &\leq 6.4\left(2K+\frac{\whReg_{\calB_{m-1}}(\pi)}{C\nu_m}\right) +160K +  \frac{\Delta_{[\tau_i,e]}}{\nu_m} \tag{Eq.~\eqref{eqn:op2}}\\
    &\leq \frac{6.4\left(2\Reg_{\calB_{m-1}}(\pi) + C_1\Kprime\nu_{m-1} + C_2\Delta_{\calB_{m-1}}\right)}{C\nu_m} + (12.8K+160K) +  \frac{\Delta_{[\tau_i,e]}}{\nu_m}  \tag{Lemma~\ref{lem:Reg gap for B_j}}\\
    &\leq \frac{\Reg_{\calA}(\pi)}{3\nu_m} + C_{10}\Kprime + \frac{2\Delta_{[\tau_i,e]}}{\nu_m}. \tag{Lemma~\ref{lem:variation and regret}}\\
    & \tag{let $C_{10}=\frac{6.4\sqrt{2}C_1}{C}+12.8+160$} 
\end{align*}
Using $U_t(\pi)\leq V_t(Q_m, \nu_m, \pi)\log_2 T$ by Lemma~\ref{lemma:variance_upperbound} and invoking Lemma~\ref{lem:useful} with \[Z=\left(C_{10}\Kprime + \frac{2\Delta_{[\tau_i,e]}}{\nu_m}\right)\log_2 T\]  finish the proof. 
\end{proof}

\section{Omitted Details in Section~\ref{sec: block reg} -- Bounding Individual Regret Terms}
\label{appendix: first two terms}
In Section~\ref{sec: block reg}, we have partitioned a block $\calJ$ into $\Gamma=\order\left(\min\{S_\calJ, 1+(KC_0)^{-1/3}\Delta_{\calJ}^{2/3}|\calJ|^{1/3}\}\right)$ intervals $\calI_1\cup \calI_2 \cup \cdots \cup \calI_\Gamma$, such that each one has $\Delta_{\calI_k}\leq \alpha_{\calI_k}$. In particular, we use the procedure described in Algorithm~\ref{alg:partition}, which only happens in the analysis, to do the partition. Then we obtain a regret bound of block $\calJ$ up to the first $\Gamma-1$ intervals as in Eq.~\eqref{eqn: term decomposition}. 

For the remaining interval $\calI_\Gamma$, because it might be interupted by restart, the terms $\alpha_{\calI_\Gamma}$ and $\varepsilon_{\calI_\Gamma}$ produced by Lemma~\ref{lem:interval regret} would depend on when we end the block, which is random and makes the analysis difficult. We resolve this issue by introducing the following \textit{fictitious block} and a new partition over it. 
\begin{definition}[fictitious block]
Define
\begin{align}
\calJ' \triangleq [\tau_i+2^{j-1}L, \tau_i+2^jL -1],     \label{eqn: Jprime definition}
\end{align}
and let $\calI'_1\cup \calI'_2, \ldots \cup \calI'_{\Upsilon}$ be a partition of $\calJ'$ using the procedure in Algorithm~\ref{alg:partition}.  
\end{definition}
Comparing the definition of $\calJ'$ to that of $\calJ$ in Eq.~\eqref{eqn: Js definition}, one can see that $\calJ$ and $\calJ'$ only differ when there is a restart triggered in block $j$. Put differently, $\calJ'$ is the \textit{planned} block $j$ while $\calJ$ is the realized block $j$. Conditioned on all history before block $j$, $\calJ'$, as well as the intervals $\calI'_1, \ldots, \calI'_{\Upsilon}$ and the excess regret and excess regret thresholds defined on them, are determined, while $\calJ$, $\calI_1, \ldots, \calI_{\Gamma}$ and similar quantities on them are random. The following facts are clear by the procedure in Algorithm~\ref{alg:partition}: 
\begin{fact}\label{fact: fictitious fact}
    Let $\{\calI_1, \calI_2, \ldots, \calI_\Gamma\}$ be the partition of $\calJ$ (defined in \eqref{eqn: Js definition}) using Algorithm~\ref{alg:partition}, and also $\{\calI_1', \calI_2', \ldots, \calI_\Upsilon'\}$ be the partition of $\calJ'$ (defined in \eqref{eqn: Jprime definition}) using the same algorithm. Then (a) $\Gamma\leq \Upsilon$,  (b) $\calI_k=\calI_k', \forall k\in [\Gamma-1]$, and (c) $s_{\calI_\Gamma}=s_{\calI_\Gamma'}, e_{\calI_\Gamma}\leq e_{\calI_\Gamma'}$, where $\calI_\Gamma\triangleq [s_{\calI_\Gamma}, e_{\calI_\Gamma}],  \calI_{\Gamma}'\triangleq [s_{\calI_\Gamma'}, e_{\calI_\Gamma'}]$. 
\end{fact}
With the above new definitions, we have the following specialized lemma for the regret in $\calI_\Gamma$, which is an analogue of Lemma~\ref{lem:interval regret}:
\begin{lemma}
With probability $1-\delta$, \algo guarantees the following for $\calI=\calI_\Gamma$ and $\calI'=\calI_{\Gamma}'$ (recall $j$ is the index of the block that contains $\calI_\Gamma$): 
\begin{align}
    &\sum_{t\in \calI} \left(r_t(\pi_t^*(x_t)) - r_t(a_t)\right)\leq \order\left(\left(\sum_{t\in\calI}\sum_{m\in M_t\cup\{j\}}\Kprime\nu_m\right) +  |\calI|\alpha_{\calI'} + |\calI|\Delta_{\calI'} + |\calI|\varepsilon_{\calI'}\one\{\varepsilon_{\calI'} > D_3\alpha_{\calI'}\}\right).  \label{eqn: term decomposition for gamma}
\end{align}
\end{lemma}
\begin{proof}
The proof is the same as Lemma~\ref{lem:interval regret}, except that we bound $\sum_{\pi\in\Pi} Q_m(\pi)\Reg_t(\pi)$ slightly differently:   
\begin{align}
    \sum_{\pi \in \Pi} Q_{m}(\pi) \Reg_t(\pi)  \nonumber 
    &\leq \sum_{\pi \in \Pi} Q_{m}(\pi) \Reg_{\calI'}(\pi) + \order(\Delta_{\calI'})   \tag{Lemma~\ref{lem:variation and regret}}   \nonumber \\
    &=  \sum_{\pi \in \Pi} 8Q_{m}(\pi)\whReg_{\calB_{j-1}}(\pi) + \order(\Delta_{\calI'}) +  \varepsilon_{\calI'} \tag{definition of $\varepsilon_{\calI'}$}   \nonumber \\
    &\leq \sum_{\pi \in \Pi} 8Q_{m}(\pi)\left(4\whReg_{\calB_{m-1}}(\pi)+ D_4\Kprime\nu_{m}\right) + \order(\Delta_{\calI'}) +  \varepsilon_{\calI'} \tag{Eq.~\eqref{eqn:test_cond4} does not hold for block $j-1$}    \nonumber \\
    &\leq  \order(\Kprime\nu_m + \Delta_{\calI'}) +  \varepsilon_{\calI'} \tag{Eq.~\eqref{eqn:op1}}  \nonumber \\
    &\leq  \order(\Kprime\nu_m + \alpha_{\calI'} + \Delta_{\calI'}) +  \varepsilon_{\calI'}\mathbf{1}\left\{\varepsilon_{\calI'} > D_3\alpha_{\calI'} \right\}. \nonumber  
\end{align}
Combining this with the rest of the proof of Lemma~\ref{lem:interval regret} finishes the proof.
\end{proof}

Using the this lemma, we write the regret in $\calI_\Gamma'$ as three terms similar to those in Eq.~\eqref{eqn: term decomposition}: 
\begin{align}
    \underbrace{ \sum_{t\in \calI_{\Gamma}} \sum_{m\in M_t\cup\{j\}}\order( \Kprime\nu_m)}_{\termprime{1}} + \underbrace{\order(|\calI_\Gamma|\alpha_{\calI_\Gamma'})}_{\termprime{2}}
    +\underbrace{\order(|\calI_\Gamma|\varepsilon_{\calI_\Gamma'}\one\{\varepsilon_{\calI_\Gamma'} > D_3\alpha_{\calI_\Gamma'}\})}_{\termprime{3}}. 
\end{align}

The rest of this section bounds the three terms $\term{1}+\termprime{1}$,
$\term{2}+\termprime{2}$, and $\term{3}+\termprime{3}$ separately.
Combining these bounds proves Lemma~\ref{lem:block_regret}.

\label{appendix: bound individual terms}
\begin{lemma}[Bounding $\term{1}+\termprime{1}$]
With probability at least $1-\delta/4$, for all block $\calJ$ with index $j$, 
\begin{align*}
\term{1}+\termprime{1}=\sum_{k=1}^\Gamma \sum_{t\in \calI_k} \sum_{m\in M_t\cup\{j\}}\Kprime\nu_m
&\leq  \otil\left(\log(1/\delta)\sqrt{KC_0 2^j L}\right). 
\end{align*}
\end{lemma}

\begin{proof}
Recall that $\calJ=\calI_1\cup\cdots \cup \calI_\Gamma$, and that there is no more than $2^{j} L$ steps in block $\calJ$. At every step with probability at most $\frac{1}{L}2^{-j/2}2^{-m/2}$ we start a replay interval with length $2^mL$. 
Therefore, with $\mathbb{I}_{t,m} \triangleq \one\{\text{the algorithm starts a replay phase of index $m$ at time $t$} \}$ we have
\begin{align}
\sum_{t\in \calJ} \sum_{m\in M_t\cup\{j\}}\Kprime\nu_m \nonumber 
&\leq 2^j L \times \Kprime \nu_j + \sum_{m=0}^{j-1} \sum_{t\in \calJ}\mathbb{I}_{t,m} \times 2^m L \times \Kprime \nu_m \nonumber \\
&= \log_2 T \left(\sqrt{KC_02^j L} + \sum_{m=0}^{j-1} \sum_{t\in \calJ} \mathbb{I}_{t,m} \sqrt{KC_0 2^m L} \right). \label{eqn: second term intermediate}
\end{align}
Note that $\sum_{t\in \calJ}\E_t\left[ \mathbb{I}_{t,m}  \right] \leq 2^{j}L \times \frac{1}{L}2^{-j/2}2^{-m/2} \leq 2^{\frac{j-m}{2}}$. 
By Freedman's inequality (Lemma~\ref{lem:freedman} with $\lambda=1$) and a union bound over all possible $\calJ$, all $j$'s and all $m$'s, we have that with probability at least $1-\delta/4$, for all possible $\calJ$ and $j$ and $m$, 
\begin{align*}
\sum_{t\in \calJ}\mathbb{I}_{t,m}
&\leq \sum_{t\in \calJ}\E_t[\mathbb{I}_{t,m}] +  \sum_{t\in \calJ} \E_t[\mathbb{I}_{t,m}^2] +  \log\frac{4T^3}{\delta}  \\
&\leq  2^{\frac{j-m}{2}+1} +  \log\frac{4T^3}{\delta} 
\end{align*}
Combining this with Eq.~\eqref{eqn: second term intermediate} and noting $j \leq \log_2 T$, we get 
\begin{align*}
\sum_{t\in \calJ} \sum_{m\in M_t\cup\{j\}}\Kprime\nu_m
&\leq \log_2 T\left( \sqrt{KC_02^j L} + (\log_2T) \mathcal{O}\left(\sqrt{KC_02^j L }\log(1/\delta)\right) \right) \\
&= \otil\left(\log(1/\delta)\sqrt{KC_0 2^j L}\right). 
\end{align*}
\end{proof}

\begin{lemma}[Bounding $\term{2}+\termprime{2}$]
For all block $\calJ$, 
\begin{align*}
\term{2}+\termprime{2}     &=\left(\sum_{k=1}^{\Gamma-1} |\calI_k| \alpha_{\calI_k}\right) + |\calI_\Gamma|\alpha_{\calI_\Gamma'} \\
&\leq \mathcal{O}\left( \log_2 T \times \min \left\{ \sqrt{KC_0S_\calJ |\calJ|}, \sqrt{KC_0|\calJ|} + (KC_0)^{\frac{1}{3}}\left(\Delta_\calJ\right)^{\frac{1}{3}} |\calJ|^{\frac{2}{3}} \right\}\right). \label{eqn: Term 2 bound}
\end{align*}
\end{lemma}
\begin{proof}
Note $\alpha_{\calI_\Gamma'}\leq \alpha_{\calI_\Gamma}$ (because $\calI_\Gamma\subseteq \calI_\Gamma'$). Therefore the left hand side is upper bounded by $\sum_{k=1}^\Gamma |\calI_k|\alpha_{\calI_k}$. 
We simply plug in the definition of $\alpha_{\calI_k}$, 
apply Cauchy-Schwarz inequality,
and use the bound on $\Gamma$ from Lemma~\ref{lemma: division} to get:
\begin{align*}
    \sum_{k=1}^{\Gamma} |\calI_k| \alpha_{\calI_k} 
    &= \log_2 T \times \sum_{k=1}^{\Gamma} \sqrt{2KC_0|\calI_k|}\leq \log_2 T \times \sqrt{2KC_0 \Gamma |\calJ|} \nonumber \\
    &\leq \mathcal{O}\left( \log_2 T \times \min \left\{ \sqrt{KC_0S_\calJ |\calJ|}, \sqrt{KC_0|\calJ|} + (KC_0)^{\frac{1}{3}}\left(\Delta_\calJ\right)^{\frac{1}{3}} |\calJ|^{\frac{2}{3}} \right\}\right). 
\end{align*}
\end{proof}



\subsection{Bounding $\term{3}+\termprime{3}$}
\label{appendix: detections}

The analysis of $\term{3}+\termprime{3}$ heavily relies on the definition of the fictitious block $\calJ'$ and the partition $\calI'_1\cup \cdots \cup \calI'_\Upsilon$ on it, which are defined at the beginning of Appendix~\ref{appendix: first two terms}.  

For an interval $\calI\subseteq \calJ'$, \term{3} only contributes to regret when the interval satisfies $\varepsilon_\calI > D_3\alpha_\calI$. These intervals have large \textit{excess regret} that causes extra regret. However, we will argue that the larger the excess regret, the sooner the algorithm can detect the non-stationarity and restart the algorithm. To prove this, we make use of the following lemma.
\begin{lemma}
\label{lem:detection}
Assume \event{1} holds. Let $\calI=[s,e]$ be an interval in the fictitious block $\calJ'$ with index $j$, and such that $\Delta_\calI\leq \alpha_\calI$ and $\varepsilon_\calI > D_3\alpha_\calI$. Then
\begin{enumerate}[label=(\alph*)]
    \item there exists an index $m_\calI\in \{0, 1, \ldots, j-1\}$ such that
    $ D_3\Kprime \nu_{m+1} < \varepsilon_\calI \leq D_3\Kprime \nu_{m} $;
    \item $|\calI| > 2^{m_\calI}L$;
    \item if the algorithm starts a replay phase $\calA$ with index $m_\calI$ within the range of $[s, e-2^{m_\calI}L]$, then the algorithm restarts when the replay phase finishes. 
\end{enumerate}
\end{lemma}

\begin{proof}
For notation simplicity, we use $m$ as shorthand for $m_\calI$. For (a), simply note that on one hand $\varepsilon_\calI\leq \max_{\pi}\Reg_\calI(\pi)\leq 1\leq D_3\Kprime\nu_0$; and on the other hand, $\varepsilon_\calI > D_3\alpha_\calI =D_3\sqrt{\frac{2KC_0}{|\calI|}}\log_2 T \geq D_3\sqrt{\frac{KC_0}{2^{j}L}}\log_2 T=D_3\Kprime\nu_j$, where the second inequality is because $|\calI|\leq |\calJ'|\leq 2^{j-1}L$. 

For (b), note that $D_3\sqrt{\frac{2KC_0}{|\calI|}}\log_2 T = D_3\alpha_\calI < \varepsilon_\calI \leq D_3\Kprime\nu_m = D_3\sqrt{\frac{KC_0}{2^{m}L}}\log_2 T$, which implies $|\calI| > 2\times 2^m L$. 

For (c), we show that the \testreplay fails when the replay phase finishes. That is, one of Eq.~\eqref{eqn:test_cond1}-Eq.~\eqref{eqn:test_cond3} will hold for some $\pi$.
Suppose for all $\pi\in\Pi$, Eq.~\eqref{eqn:test_cond2} and Eq.~\eqref{eqn:test_cond3} do not hold, then by Lemma~\ref{lem:2 and 3 hold} we have for all $\pi$, 
   \begin{align*}
       \Reg_\calA(\pi) &\leq 2\whReg_\calA(\pi) + C_3\Kprime\nu_m + C_4\Delta_\calA \\
       &\leq 2\whReg_\calA(\pi) + C_3\Kprime\nu_m + C_4\Delta_\calI \tag{$\calA$ lies in $\calI$} \\
       &\leq 2\whReg_\calA(\pi) + C_3\Kprime\nu_m + C_4\alpha_\calI \tag{by the condition} \\
       &\leq 2\whReg_\calA(\pi) + C_3\Kprime\nu_m + C_4\Kprime\nu_m \tag{$D_3\alpha_\calI < \varepsilon_\calI \leq D_3\Kprime\nu_m$} \\
       &= 2\whReg_\calA(\pi) + (C_3+C_4)\Kprime\nu_m. 
   \end{align*}
   Also by the definition of excess regret, there exists $\pi'$ such that 
   \begin{align*}
       \Reg_\calA(\pi') &\geq \Reg_\calI(\pi') - 2\Delta_\calI \tag{Lemma~\ref{lem:variation and regret}} \\
       &\geq 8\whReg_{\calB_{j-1}}(\pi') + \varepsilon_\calI - 2\alpha_\calI \\
       &\geq 8\whReg_{\calB_{j-1}}(\pi') + D_3\Kprime \nu_{m+1} - 2\Kprime\nu_m \\
       &\geq 8\whReg_{\calB_{j-1}}(\pi') + (0.5D_3-2) \Kprime\nu_m. 
   \end{align*}
   Combining the two inequalities we get 
   \begin{align*}
       \whReg_\calA(\pi') > 4\whReg_{B_{j-1}}(\pi') + \frac{0.5D_3-2-C_3-C_4}{2}\Kprime\nu_m > 4\whReg_{B_{j-1}}(\pi') +  D_1\Kprime\nu_m,  
   \end{align*}
   which makes Eq.~\eqref{eqn:test_cond1} hold.
\end{proof}
Now we  are ready to bound $\term{3}+\termprime{3}$:
\begin{lemma}[Bounding $\term{3}+\termprime{3}$]
With probability at least $1-\delta$, 
    \begin{align*}
        \term{3}+\termprime{3}
& =\sum_{k=1}^{\Gamma-1} |\calI_k|\varepsilon_{\calI_k}\one\{\varepsilon_{\calI_k} > D_3\alpha_{\calI_k}\} + |\calI_\Gamma|\varepsilon_{\calI_\Gamma'} \one\{\varepsilon_{\calI'_\Gamma} > D_3\alpha_{\calI'_\Gamma}\} \\
& \leq  \order\left( \log(1/\delta)\log(T)\sqrt{KC_0\Gamma 2^j L}  \right) \\
& \leq  \order\left( \log(1/\delta)\log(T) \min\left\{ \sqrt{KC_0S_\calJ 2^j L}, \sqrt{KC_0 2^jL} + (KC_0)^{\frac{1}{3}}\Delta_{\calJ}^{\frac{1}{3}}(2^j L)^{\frac{2}{3}}  \right\} \right). 
    \end{align*}
\end{lemma}

\begin{proof}
We condition on the history before block $j$. As discussed at the beginning of Appendix~\ref{appendix: first two terms}, under this condition, the partition $\calI_1', \ldots, \calI_{\Upsilon}'$, as well as excess regret and excess regret thresholds defined on them are all fixed. 
Define $\calK=\{k\in [\Upsilon]~|~\varepsilon_{\calI_k'} > D_3\alpha_{\calI_k'}\}$. Then by Fact~\ref{fact: fictitious fact}$, \term{3}+\termprime{3}$ can be written as 
\begin{align*}
\sum_{k\in [\Gamma]} |\calI_k|\varepsilon_{\calI_k'}\one\{\varepsilon_{\calI_k'}>D_3\alpha_{\calI_k'}\} =  \sum_{k\in [\Gamma] \cap \calK} |\calI_k| \varepsilon_{\calI_k'}. 
\end{align*}
%
For $k\in \calK$, denote by $m_k$ the index $m_{\calI_k'}$ defined in Lemma~\ref{lem:detection}. Then
\begin{align*}
\sum_{k\in [\Gamma] \cap \calK} |\calI_k| \varepsilon_{\calI_k'}
&\leq \sum_{k\in [\Gamma] \cap \calK} |\calI_k| D_3 K\nu_{m_k}\\
&= \sum_{k\in [\Gamma] \cap \calK} \left( 2^{m_k}L \times D_3\Kprime \nu_{m_k}+ (|\calI_k|-2^{m_k}L) \times D_3\Kprime \nu_{m_k}  \right) \\
&\leq \underbrace{ \sum_{k\in [\Gamma] \cap \calK}(\log_2 T)D_3\sqrt{KC_0 2^{m_k}L}}_{\term{4}} + \underbrace{\sum_{k\in[\Gamma] \cap \calK} (|\calI_k|-2^{m_k}L) D_3 \Kprime\nu_{m_k}}_{\term{5}}.
\end{align*}
\term{4} can be bounded as
\begin{align*}
\sum_{k\in [\Gamma] \cap \calK}(\log_2 T)D_3\sqrt{KC_0 2^{m_k}L}
&\leq \sum_{k\in [\Gamma] \cap \calK} (\log_2 T)D_3\sqrt{KC_0 |\calI_k'|}   \tag{Lemma~\ref{lem:detection}, (b)}\\
&\le (\log_2 T)D_3\sqrt{KC_0\Gamma \sum_{k\in[\Gamma] \cap \calK}|\calI_k'|}   \tag{Cauchy-Schwarz}\\
&= \order\left((\log T)\sqrt{KC_0\Gamma 2^j L}\right).  
\end{align*}
Next we want to bound \term{5}. Rewrite it as the following (denote $\calI_k=[s_k, e_k], \calI_k'=[s_k', e_k']$): 
\begin{align*}
\term{5} &\leq \sum_{k\in[\Gamma]\cap \calK}\sum_{t\in[s_k+2^{m_k}L, e_k]}D_3 \Kprime \nu_{m_k} = \sum_{k\in \calK}\sum_{t\in[s_k+2^{m_k}L, e_k]}D_3 \Kprime \nu_{m_k} \one\{t\leq e_\Gamma\} \\
&=\sum_{k\in \calK}\sum_{t\in[s_k'+2^{m_k}L, e_k']}D_3 \Kprime \nu_{m_k} \one\{t\leq e_\Gamma\}
\end{align*}
Below we show how to upper bound for a number $z$ the probability $\Pr\{\term{5}> z\}$.  
Define the following function:
\begin{align*}
f(\tau)\triangleq\sum_{k\in \mathcal{K}}\sum_{t\in[s_k'+2^{m_k}L,e_k']} D_3 \Kprime\nu_{m_k}\one\{t\le\tau \},
\end{align*}
which is again not random conditioning on the history before block $j$.
Our strategy is to first bound $\Pr\{\term{5}>f(\tau)\}$ by some function of $f(\tau)$. First by comparing definitions it is clear that $\Pr\{\term{5}>f(\tau)\}\leq \Pr\{e_\Gamma > \tau\}$.  Note that $e_\Gamma > \tau$ is the event that the block ends later than $\tau$. 
From Lemma~\ref{lem:detection}, we know that in interval $\calI_k'$ with $k\in \calK$, if the algorithm starts an replay phase at some $t^*\in [s_k, e_k-2^{m_k}L]$, then conditioned on \event{1}, the algorithm will restart at (or before) $t^*+2^{m_k}L$. That is, for an intervals $k\in\mathcal{K}$, the algorithm always has $|\calI_k'|-2^{m_k}L$ opportunities to start a replay phase with index $m_k$ which triggers restart eventually. Thus, if the algorithm proceeds to time $\tau$, and has not restarted yet, it has missed all such opportunities before $\tau$. More precisely, for all $k\in\calK$ with $e_k' < \tau$ (i.e., the intervals that are before $\tau$), the algorithms misses all opportunities in 
\begin{align*}
[s_k', e_k'-2^{m_k}L]
\end{align*}
to start a replay phase with index $m_k$; for $k$ such that $\tau\in[s_k', e_k']$ (i.e., the interval where $\tau$ lies in), the algorithm misses all opportunities in 
\begin{align*}
[s_k', \tau-2^{m_k}L]
\end{align*}
to start a replay phase with index $m_k$ (define this set to be empty if $\tau-2^{m_k}L<s_k'$). Combining these cases, we see that the probability that the block ends later than $\tau$ (i.e., $e_\Gamma>\tau$) is smaller than 
\begin{align*}
    &\prod_{k\in\mathcal{K}}\prod_{t\in[s_k',e_k'-2^{m_k}L]} (1-q_{m_k}\one\{t\le\tau-2^{m_k}L\}),\\  
    &=\prod_{k\in\mathcal{K}}\prod_{t\in[s_k'+2^{m_k}L,e_k']} (1-q_{m_k}\one\{t\le\tau\}). 
\end{align*}
where $q_{m}=\frac{1}{L}\sqrt{\frac{1}{2^j2^{m}}}=\sqrt{\frac{K}{C_02^jL}}\nu_{m}$ is the probability to start a replay phase with index $m$ at any time. 
Using the inequality $1-x\leq e^{-x}$, the above probability can further be upper bounded by
\begin{align*}
&\prod_{k\in\mathcal{K}}\prod_{t\in[s_k'+2^{m_k}L,e_k']} \exp\left( -q_{m_k}\one\{t\le\tau\}) \right) \\
&= \exp\left(- \sum_{k\in\mathcal{K}}\sum_{t\in[s_k'+2^{m_k}L,e_k']} q_{m_k}\one\{t\le\tau\} \right) \\
&= \exp\left(- \sum_{k\in\mathcal{K}}\sum_{t\in[s_k'+2^{m_k}L,e_k']} \sqrt{\frac{K}{C_02^jL}}\nu_{m_k}\one\{t\le\tau\} \right) \\
&= \exp\left( -\sum_{k\in\mathcal{K}}\sum_{t\in[s_k'+2^{m_k}L,e_k']} D_3\Kprime \nu_{m_k} \one\{t\le\tau\}\cdot\frac{1}{C^*} \right) \\
&=  \exp\left(  -\frac{f(\tau)}{C^*} \right), 
\end{align*}
where $C^*\triangleq (\log_2 T)D_3\sqrt{KC_02^j L}$. Combining all arguments above, we have
\begin{align*}
\Pr\{\term{5} > f(\tau)\} \leq \Pr\{e_\Gamma > \tau\} \leq \exp\left( -\frac{f(\tau)}{C^*} \right). 
\end{align*}
Picking $z=\log(e/\delta)C^*$, we then let $\tau$ be such that $f(\tau)\le z<f(\tau+1)$. If no such $\tau$ exists, $\Pr\left[\term{5} > z\right]=0$; otherwise, since $f(\tau)\geq f(\tau+1) - D_3(\log_2 T) > z-D_3(\log_2 T)$, we have

\begin{align*}
\Pr\left[\term{5}> z\right]
&\le\Pr\left[\term{5}> f(\tau)\right]
\le \exp\left(-\frac{f(\tau)}{C^*}\right)
\le\exp\left(\frac{D_3\log_2 T}{C^*}-\frac{z}{C^*} \right) \\
&\leq e\times \frac{\delta}{e}=\delta. 
\end{align*}
Therefore, we conclude that $\term{5}\ge\log(e/\delta)C^*=\log(e/\delta)(\log_2 T)D_3\sqrt{KC_02^j L}$ with probability at most $\delta$. 
Combining \term{4} and \term{5}, we get the first bound. Further using Lemma~\ref{lemma: division} gives the second bound. 
\end{proof}

\section{Bounding the Number of Restarts}
\label{appendix:restart}
\begin{lemma}
\label{lem:no rerun until}
Assume \event{1} holds. Then for all $t$ in epoch $i$ with $\Delta_{[\tau_i,t]}\leq \sqrt{\frac{KC_0}{t-\tau_i+1}}$, restart will not be triggered at time $t$. 
\end{lemma}

\begin{proof}
We will show that both tests pass and thus the algorithm does not restart.

\paragraph{No restarts by \testblock.}
Let $t=\tau_i+2^j L-1$ for some $j$. Then $\Delta_{[\tau_i,t]}\leq \sqrt{\frac{KC_0}{t-\tau_i+1}} = K\nu_j$. For any $\pi\in\Pi, k\in[0,j-1]$, 
\begin{align*}
    \whReg_{\calB_j} 
    &\leq 2\Reg_{\calB_j} + C_1\Kprime\nu_j + C_2\Delta_{[\tau_i,t]} \tag{Lemma~\ref{lem:Reg gap for B_j}}\\
    &\leq 2\Reg_{\calB_k} + C_1\Kprime\nu_j + (C_2+4)\Delta_{[\tau_i,t]} \tag{Lemma~\ref{lem:variation and regret}}\\
    &\leq 2\left(2\whReg_{\calB_k} + C_1\Kprime\nu_k + C_2\Delta_{[\tau_i,t]}\right) + C_1\Kprime\nu_j + (C_2+4)\Delta_{[\tau_i,t]} \tag{Lemma~\ref{lem:Reg gap for B_j}}\\
    & \leq 4\whReg_{\calB_k} + 3C_1\Kprime\nu_k + (3C_2+4)\Delta_{[\tau_i,t]} \\
    &\leq 4\whReg_{\calB_k} + D_4\Kprime\nu_k. \tag{$3C_1+3C_2+4\leq D_4$}
\end{align*}
Similarly, $\whReg_{\calB_k}\leq 4\whReg_{\calB_j} + D_4\Kprime\nu_k$. On the other hand, 
\begin{align*}
    \whV_{\calB_{k}}(Q_{k+1}, \nu_{k+1}, \pi) 
    &\leq 6.4V_{\calB_k}(Q_{k+1}, \nu_{k+1}, \pi) +\frac{80C_0}{\nu_{k+1}^2|\calB_k|}  \tag{\event{1}} \\
    &\leq 6.4V_{\calB_j}(Q_{k+1}, \nu_{k+1}, \pi) + \frac{80C_0}{\nu_{k+1}^2|\calB_k|}  + \frac{6.4\Delta_{[\tau_i,t]}}{\nu_{k+1}} \tag{Lemma~\ref{lem:variation and TVD}} \\
    &\leq 6.4\left(6.4\whV_{\calB_j}(Q_{k+1}, \nu_{k+1}, \pi) + \frac{80C_0}{\nu_{k+1}^2|\calB_j|}\right) + \frac{80C_0}{\nu_{k+1}^2|\calB_k|}  + \frac{6.4\Delta_{[\tau_i,t]}}{\nu_{k+1}} \tag{\event{1}}\\
    &\leq 41\whV_{\calB_j}(Q_{k+1}, \nu_{k+1}, \pi) + D_5K. \tag{$6.4\times  80 +  160 + 6.4 \leq D_5$}
\end{align*}

Therefore, Eq.~\eqref{eqn:test_cond4}-\eqref{eqn:test_cond6} do not hold for all $\pi$ and all $k\in[0,j-1]$ and the algorithm will not restart. 

\paragraph{No restarts by \testreplay.}
Let $\calA\subseteq [\tau_i,t]$ be a complete replay interval of index $m$. Then $\Delta_{[\tau_i,t]}\leq \sqrt{\frac{KC_0}{t-\tau_i+1}}\leq \sqrt{\frac{KC_0}{|\calA|}} = K\nu_m$. We have the following:  
\begin{align*}
    \whReg_{\calA}(\pi) 
    &\leq 2\Reg_{\calA}(\pi) + C_5\Kprime\nu_m + C_6\Delta_{[\tau_i,t]}  \tag{Lemma~\ref{lem: regret relation with [1,e]}}\\
    &\leq 2\Reg_{\calB_{j-1}}(\pi) + C_5\Kprime\nu_m + (C_6+4)\Delta_{[\tau_i,t]} \tag{Lemma~\ref{lem:variation and regret}}\\
    &\leq 2\left(2\whReg_{\calB_{j-1}}(\pi)+ C_5\Kprime\nu_{j-1}+C_6\Delta_{[\tau_i,t]}\right) + C_5\Kprime\nu_m + (C_6+4)\Delta_{[\tau_i,t]} \tag{Lemma~\ref{lem:Reg gap for B_j}}\\
    &\leq 4\whReg_{\calB_{j-1}}(\pi) + D_1\Kprime\nu_m.   \tag{$3C_5+3C_6+4\leq D_1$}
\end{align*}
Similarly, $\whReg_{\calB_{j-1}}(\pi)\leq 4\whReg_{\calA}(\pi) + D_1\Kprime\nu_m$. Also, 
\begin{align*}
    \whV_{\calA}(Q_m, \nu_m, \pi) &\leq 6.4V_{\calA}(Q_m, \nu_m, \pi) + \frac{80C_0}{\nu_m^2 |\calA|} \tag{\event{1}} \\
    &\leq 6.4V_{\calB_{j-1}}(Q_m, \nu_m, \pi) + 80K + \frac{6.4\Delta_{[\tau_i,t]}}{\nu_m} \tag{Lemma~\ref{lem:variation and TVD}}\\
    &\leq 6.4\left(6.4\whV_{\calB_{j-1}}(Q_m, \nu_m, \pi) + \frac{80C_0}{\nu_m^2|\calB_{j-1}|}\right) + 80K + \frac{6.4\Delta_{[\tau_i,t]}}{\nu_m} \tag{\event{1}}\\
    &\leq 41 \whV_{\calB_{j-1}}(Q_m, \nu_m, \pi) + D_2K. \tag{$6.4\times 80+80+6.4\leq D_2$}
\end{align*}
Therefore Eq.~\eqref{eqn:test_cond1}-\eqref{eqn:test_cond3} do not hold for all $\pi$ and the algorithm will not restart. 
\end{proof}

\begin{proof}[Lemma~\ref{lemma: number of restarts}]
Condition on \event{1}, which happens with probability $1-\delta/2$.
By Lemma~\ref{lem:no rerun until}, if there is no distribution change since last restart (which implies $\Delta_{[\tau_i,t]}=0$), then the algorithm will not restart again. Thus $E\leq S$. 

Let the epoch length be $T_1, \ldots, T_{E}$. Again by Lemma~\ref{lem:no rerun until}, in epoch $i$, the total variation has to be larger than $\sqrt{\frac{KC_0}{T_i}}$. By H\"older's inequality, we have
\begin{align*}
    E-1 \leq \left(\sum_{i=1}^{E-1} T_i\right)^{\frac{1}{3}}\left(\sum_{i=1}^{E-1} \sqrt{\frac{1}{T_i}} \right)^{\frac{2}{3}} \leq T^{\frac{1}{3}} 
    \left(\frac{\Delta}{\sqrt{KC_0}}\right)^{\frac{2}{3}} = (KC_0)^{-\frac{1}{3}}\Delta^{\frac{2}{3}}T^{\frac{1}{3}}. 
\end{align*}
This finishes the proof.
\end{proof}

\section{Notation Table}
\label{app:notation_table}

\begin{table}[htbp]\caption{General Notations}
\begin{center}
\begin{tabular}{r c p{10cm} }
\toprule
Notation & & Meaning \\
\hline
\hline
$\calX$  & ~ & context space \\
$K$ & ~ & number of actions \\
$\calD_t$  &~ & the density function over $\calX\times [0,1]^K$ at time $t$\\
$\calD_t^\calX$  &~ & the marginal distribution of $\calD_t$ over the context space $\calX$\\
$(x_t, r_t)$ &~& context-reward pair drawn from $\calD_t$ at time $t$\\
$S$ & ~ & $1+\sum_{t=2}^T \one\{\calD_{t}\neq \calD_{t-1}\}$ \\
$\var$ & ~ & $\sum_{t=2}^T \TVD{\calD_t-\calD_{t-1}}$ \\
$\calI = [s,s']$& ~ & an time interval consisting of time steps $s,s+1, \ldots, s'$\\
$S_{[s,s']}$ & ~ & $1+\sum_{\tau=s+1}^{s'}\one\{\calD_{\tau}\neq \calD_{\tau-1}\}$\\
$\Delta_{[s,s']}$ & ~ & $\sum_{\tau=s+1}^{s'} \TVD{\calD_{\tau}-\calD_{\tau-1}}$\\
$\calR_t(\pi)$ & ~ & $\E_{(x,r)\sim \calD_t}\left[r(\pi(x))\right]$\\
$\pi_t^*$& ~ & $\argmax_{\pi\in \Pi}\calR_t(\pi)$\\
$\calR_\calI(\pi)$ & ~ & $\frac{1}{|\calI|}\sum_{t\in\calI} \calR_t(\pi)$\\
$\pi_\calI^*$ & ~ & $\argmax_{\pi\in\Pi} \calR_\calI(\pi)$\\
$\Reg_\calI(\pi)$ & ~ & $ \calR_{\calI}(\pi_\calI^*) - \calR_\calI(\pi)$\\ 
$\ips_t(a)$ & ~ & $\frac{r_t(a)}{p_t(a)}\one\{a_t=a\}$ where $a_t \sim p_t$ is the action selected at time $t$\\
$\avgR_\calI(\pi)$ & ~ & $\frac{1}{|\calI|}\sum_{t\in\calI} \ips_t(\pi(x_t))$\\
$\whpi_\calI$ & ~ & $\argmax_{\pi\in\Pi} \avgR_\calI(\pi)$.\\
$\whReg_\calI(\pi)$ & ~ & $ \avgR_\calI(\whpi_\calI) - \avgR_\calI(\pi)$\\
$Q \in \Delta^\Pi$ & ~ & $\{Q \in \fR^{|\Pi|}_+: \sum_{\pi \in \Pi} Q(\pi) =
1\}$\\
$Q(a|x)$ & $ ~ $ & $\sum_{\pi: \pi(x)=a} Q(\pi)$\\
$Q^\nu(\cdot|x)$ & $ ~ $ & $\nu\one +(1-K\nu)Q(\cdot|x)$\\
$\whV_\calI (Q, \nu, \pi)$ & ~ & $ \frac{1}{|\calI|} \sum_{t \in \calI} \left[ \frac{1}{Q^{\nu}(\pi(x_t) | x_t)} \right]$\\ 
$V_\calI(Q, \nu, \pi)$ & ~ & $ \frac{1}{|\calI|} \sum_{t \in \calI} \E_{x \sim \calD_t^\calX} \left[ \frac{1}{Q^{\nu}(\pi(x) | x)} \right]$\\
$\Kprime$ & ~ & $(\log_2T)K$\\
 
\bottomrule
\end{tabular}
\end{center}
\label{tab:TableOfNotationForMyResearch}
\end{table}

\end{document}